\documentclass[11pt]{article}

\usepackage[margin=1in]{geometry}
\usepackage[numbers]{natbib}

\usepackage[utf8]{inputenc}
\usepackage[T1]{fontenc}
\usepackage{url}
\usepackage{booktabs}
\usepackage{multirow}
\usepackage{amsfonts}
\usepackage{amssymb}
\usepackage{amsmath}
\usepackage{amsthm}
\usepackage{nicefrac}
\usepackage{microtype}
\usepackage[dvipsnames]{xcolor}
\usepackage{graphicx}
\usepackage{tikz}
\usepackage{pgfplots}
\pgfplotsset{compat=1.18}
\usepackage{float}
\usepackage{algorithm}
\usepackage{algpseudocode}
\usepackage{dsfont}
\usepackage{etoc}
\usepackage{listings}
\lstdefinestyle{clusterdump}{%
  basicstyle=\ttfamily\scriptsize,
  breaklines=true,
  breakatwhitespace=false,
  columns=fullflexible,
  keepspaces=true,
  frame=single,
  framesep=3pt,
  xleftmargin=2pt,
  showstringspaces=false,
  postbreak=\mbox{\textcolor{gray}{$\hookrightarrow$}\space},
}

\newtheorem{theorem}{Theorem}
\newtheorem{proposition}[theorem]{Proposition}
\newtheorem{lemma}[theorem]{Lemma}
\newtheorem{corollary}[theorem]{Corollary}
\theoremstyle{definition}
\newtheorem{definition}[theorem]{Definition}
\newtheorem{remark}[theorem]{Remark}

\usepackage[hidelinks,hypertexnames=false]{hyperref}

\title{Frustratingly Simple Black-Box Adaptation of Language Models via Logit Bias}

\author{%
  Ofek I. Cohen \\
  Tel Aviv University \\
  \texttt{ofeki@mail.tau.ac.il}
  \and
  Lior Shani \\
  Google Research \\
  \texttt{liorshani@google.com}
  \and
  Aviv Rosenberg \\
  Google Research \\
  \texttt{avivros@google.com}
  \and
  Ankur Samanta \\
  Columbia University \\
  \texttt{as7416@columbia.edu}
  \and
  Tal Wagner \\
  Tel Aviv University \\
  \texttt{talwag@tauex.tau.ac.il}
  \and
  Yonathan Efroni \\
  Tel Aviv University \\
  \texttt{yefroni@tauex.tau.ac.il}
}

\date{}

\begin{document}
\setlength{\parskip}{0pt}

\maketitle
\etocdepthtag.toc{mainbody}

\begin{abstract}
Many organizations aim to adapt language models for internal use, both to improve performance on domain-specific tasks and to address privacy concerns around sensitive data. However, such adaptation remains non-trivial: it often requires operationally challenging fine-tuning of open-source models or ad hoc prompt optimization. We study a minimal alternative based on a simple API-level control: allowing users to bias the model’s logits with a user-defined vector. We develop a black-box method for learning a single context-independent logit-bias vector, added at every decoding step, without modifying model weights or requiring gradients. Starting from a KL-regularized reinforcement learning (RL) objective, we characterize when such a fixed logit-bias vector can approximate the optimal prefix-dependent correction and derive a closed-form inverse-propensity estimator from rollouts, rewards, and token probabilities. Empirically, this simple decoding-time intervention improves over base models on mathematical and reasoning benchmarks while using far fewer trainable parameters than conventional fine-tuning. Our results suggest that learned logit bias is a lightweight mechanism for adapting language models under minimal access requirements.

\end{abstract}

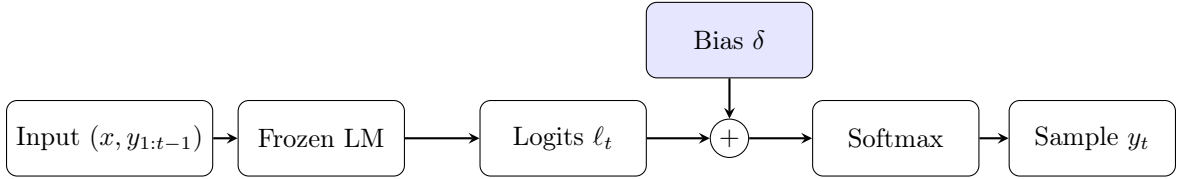
\begin{figure}[t]
  \centering
  \begin{tikzpicture}[
    box/.style={draw, rounded corners, minimum height=1cm, minimum width=2.2cm,
                align=center, font=\small},
    arrow/.style={->, thick, >=stealth},
    plus/.style={circle, draw, inner sep=1.5pt, font=\small\bfseries},
  ]
    \node[box] (ht) at (-2.8,0) {Input $(x,y_{1:t-1})$};
    \node[box] (lm)   at (0,0)    {Frozen LM};
    \node[box] (logits) at (3.2,0) {Logits $\ell_t$};
    \node[plus] (add) at (5.4,0) {$+$};
    \node[box, fill=blue!10] (delta) at (5.4,1.3) {Bias $\delta$};
    \node[box] (sm) at (7.6,0) {Softmax};
    \node[box] (sample) at (10.2,0) {Sample $y_t$};

    \draw[arrow] (ht) -- (lm);
    \draw[arrow] (lm) -- (logits);
    \draw[arrow] (logits) -- (add);
    \draw[arrow] (delta) -- (add);
    \draw[arrow] (add) -- (sm);
    \draw[arrow] (sm) -- (sample);
  \end{tikzpicture}
  \caption{Overview of logit-bias steering. The fixed logit-bias vector~$\delta$ is added to the base model's logits at every decoding step before sampling. Only~$\delta$ is optimized; model weights remain unchanged.}
  \label{fig:method}
\end{figure}

\section{Introduction}
\label{sec:intro}

Deploying frontier language models often requires adapting them to narrow operational needs: shorter reasoning traces, stricter output formatting, or improved accuracy on small, confidential in-house distributions~\citep{salemi2024lamp,salemi2024privacy,koo2024automata,willard2023efficient,beurerkellner2024guiding}. In many practical settings, however, the model is available only through an inference API: the deployer cannot inspect or modify weights, cannot backpropagate through the model, and must work with the controls exposed at decoding time~\citep{sun2022blackbox,sun2022bbtv2,deng2022rlprompt,hiranandani2025token}. This makes standard adaptation---fine-tuning, PEFT, prompt embeddings, RLHF-style optimization~\citep{houlsby2019adapters,li2021prefix,lester2021power,hu2021lora,benzaken2022bitfit,liu2022ia3,ziegler2019fine,ouyang2022training}---unavailable or mismatched to the setting.

\begin{center}
    \textit{Can a single learned logit-bias vector provide useful black-box adaptation?}
\end{center}

We study logit bias as a minimal adaptation primitive. Token-level logit-bias controls are documented in major inference stacks and hosted APIs~\citep{vllm_logit_bias,sglang_logit_bias,lmstudio_openai_compat,llamacpp_logit_bias,openai_logit_bias,azure_openai_logit_bias,fireworks_logit_bias}. We consider the simplest possible intervention: starting from a frozen base model with no learned logit bias, add a single context-independent logit-bias vector $\delta$ over the vocabulary to the model's logits at every decoding step. The base model's weights are never modified; only $\delta$ is learned, using black-box rollouts rather than gradients or access to model internals.

This intervention is deliberately weak. The fixed logit-bias vector $\delta$ cannot encode arbitrary prompt- or prefix-dependent behavior, unlike prompts that can express explicit conditional instructions~\citep{openai_prompt_engineering,azure_prompt_engineering}. But this weakness is also what makes the primitive attractive: when the desired correction is approximately token-specific and stable across contexts, a single vector may capture a useful slice of adaptation with little serving overhead. A particularly relevant case is verbosity control: although $\delta$ cannot decide which reasoning steps are mathematically necessary, it may still learn reusable finalization, formatting, and stopping cues that shorten completions without changing the model weights. Unlike prompt steering it consumes no context length and acts directly on the output distribution, and unlike weight-space adaptation it requires no gradients, hidden states, or provider-side training access~\citep{openai_api_pricing,liu2024lostmiddle,houlsby2019adapters,hu2021lora}.

We formalize this via a KL-regularized RL objective~\citep{ziegler2019fine,stiennon2020summarize,ouyang2022training,korbak2022rl,rafailov2023direct,zhou2025qsharp}; the unconstrained optimum induces a prefix-dependent correction that we approximate with the fixed logit-bias vector $\delta$, yielding a reward-gap bound and a closed-form IPS estimator from rollouts, rewards, and token probabilities~\citep{horvitz1952generalization}.

\paragraph{Contributions.}
We make three contributions.\footnote{Code available at \url{https://github.com/Ofek-Israeli/logit_bias_lm_adaptation}.} First, we introduce \emph{black-box logit-bias adaptation}, a setting for personalizing frozen language models with no weight updates and no gradient access. Second, we derive a closed-form estimator and a theoretical analysis connecting logit-bias steering to the KL-regularized optimum. Third, we show empirically that full-vocabulary logit bias improves base models on mathematical and reasoning benchmarks while using far fewer trainable parameters and far less compute than conventional fine-tuning.

\section{Related work}
{\setlength{\parskip}{0pt}

\paragraph{Black-box or API-only adaptation.}
Existing black-box methods adapt through the input channel, using derivative-free continuous prompts or reward-learned discrete prompts~\citep{sun2022blackbox,sun2022bbtv2,deng2022rlprompt}; such prompt-based control consumes context and can interfere with task content or learned prompting behavior~\citep{li2021prefix,lester2021power,liu2024lostmiddle,deng2022rlprompt}. We instead adapt through the output channel: a single context-independent logit-bias vector learned from rollouts and applied at decoding~\citep{hiranandani2025token}, which is less expressive but easy to store, audit, and deploy via standard decoding controls, with no gradients or provider-side fine-tuning~\citep{openai_logit_bias,azure_openai_logit_bias,vllm_logit_bias}.

\paragraph{Decoding-time control and logit guidance.}
Decoding-time methods steer generation by perturbing hidden states, reweighting next-token probabilities with auxiliary models, enforcing hard constraints, or applying generic logits processors~\citep{dathathri2020pplm,yang2021fudge,krause2021gedi,liu2021dexperts,lu2021neurologic,geng2023grammar,wolf2020transformers}; our method is a simpler member of this family, using no hidden-state gradients, auxiliary model, hard grammar, or per-prefix optimization---only one learned logit-bias vector reused at every step.

\paragraph{RLHF, KL-regularized objectives, and DPO.}
The alignment literature provides our theoretical backdrop. \citet{ziegler2019fine,stiennon2020summarize,ouyang2022training} established the standard pipeline of preference data, reward modeling, and policy optimization with a proximity penalty to the base model. \citet{korbak2022rl} cast KL-regularized RL for language models as variational/Bayesian inference, and DPO gives a closed-form preference-to-policy relation under a KL constraint~\citep{rafailov2023direct}. Our theory is closest to this KL-tilted view: we start from a reward-tilted target close to the base model, but instead of training a new policy or reward model, ask what improvement survives after projecting that target onto the severely restricted family of a single context-independent logit-bias vector.

\paragraph{Parameter-efficient and prompt-based adaptation.}
Parameter-efficient adaptation reduces cost by updating a small subset of parameters or learned prompts. Representative methods include adapters~\citep{houlsby2019adapters}, prefix-tuning~\citep{li2021prefix}, prompt tuning~\citep{lester2021power}, LoRA~\citep{hu2021lora}, BitFit~\citep{benzaken2022bitfit}, and $(IA)^3$~\citep{liu2022ia3}. They show how far one can go with limited trainable state, but still need training-time access to weights, gradients, hidden states, or prompt embeddings. Our setting forbids this access: the base model stays frozen and adaptation is restricted to the fixed logit-bias vector $\delta$ applied at decoding.
}

\section{Preliminaries}
\label{sec:preliminaries}

We consider autoregressive generation from a language model with a finite vocabulary~$\mathbb{V}$.
Let $\rho$ denote a distribution over a prompt space~$\mathcal{X}$ and let $x\sim\rho$ be a prompt drawn from it.
For a given prompt~$x$, generation proceeds for a fixed horizon $T\in\mathbb{N}$ tokens.%
\footnote{In practice the horizon varies per sequence due to early stopping at an EOS token.
We use a fixed~$T$ to simplify the theoretical framework; all results extend to variable-length generation by treating EOS as an absorbing state.}
At each step~$t$, the base language model receives the prompt and the tokens generated so far, $y_{1:t-1}=(y_1,\dots,y_{t-1})$, and produces logits $\ell(x,y_{1:t-1})\in\mathbb{R}^{|\mathbb{V}|}$.
The base next-token distribution is 
$\pi_0(y_t\mid x,y_{1:t-1}) := \operatorname{softmax}(\ell(x,y_{1:t-1}))_{y_t}$ with
$
\operatorname{softmax}(v)_y 
:=
\frac{\exp(v_y)}{\sum_{y'\in \mathbb{V}}\exp(v_{y'})}$, $y\in \mathbb{V}$,
and the trajectory law is $p_0(y_{1:T}\mid x) := \prod_{t=1}^{T}\pi_0(y_t\mid x,y_{1:t-1})$.
For a logit-bias vector $\delta\in\mathbb{R}^{|\mathbb{V}|}$, the \emph{fixed-bias policy} is $\pi_\delta(y_t\mid x,y_{1:t-1}) := \operatorname{softmax}(\ell(x,y_{1:t-1})+\delta)_{y_t}$, with trajectory law $p_\delta(y_{1:T}\mid x) := \prod_{t=1}^{T}\pi_\delta(y_t\mid x,y_{1:t-1})$.
We assume full support: $\pi_0(y_t\mid x,y_{1:t-1})>0$ for every reachable prefix~$y_{1:t-1}$ and every token $y_t\in \mathbb{V}$.
This holds automatically for any softmax-based model, since softmax outputs are strictly positive.
We use ``trajectory'' for the token sequence $y_{1:T}$ in the theoretical development and ``completion'' for the same object in experimental and evaluation contexts; both names refer to the same generated token sequence.
Inside expectations and variances we write $Y$ for the random trajectory drawn from the relevant law and $y_{1:T}$ for a realized token sequence.

\paragraph{Intervention notation.}
We write $Y\sim p_0\bigl(\cdot\mid x,\,\mathrm{do}(y_t=y)\bigr)$ to denote sampling a trajectory under
$p_0(\cdot\mid x)$ with the $t$-th token set to~$y$ by \emph{intervention} rather than 
conditioning~\citep{pearl2009causality}: $y_{1:t-1}\sim p_0(\cdot\mid x)$, $y_t=y$ is fixed, 
and $y_{t+1:T}\sim p_0(\cdot\mid x,y_{1:t-1},y)$.
This differs from the conditional distribution $p_0(\cdot\mid x,y_t=y)$, which Bayes-reweights the 
prefix by $\pi_0(y\mid x,y_{1:t-1})$.

Given a terminal reward function $r:\mathcal{X}\times\mathbb{V}^T\to\mathbb{R}$ and a prompt-conditional trajectory law~$q(\cdot\mid x)$ on~$\mathbb{V}^T$, we define the reward functional
\[
  J(q) \;:=\; \mathbb{E}_{x\sim\rho}\!\left[\mathbb{E}_{Y\sim q(\cdot\mid x)}\!\bigl[r(x,Y)\bigr]\right],
\]
Throughout, $\widehat{\mathbb{E}}$ denotes an empirical average over samples.

\paragraph{RLVR setting and accuracy reward.}
We primarily consider RL with verifiable rewards (RLVR): each prompt $x$ has a canonical correct answer $m_{\star}$, and a generated completion $y_{1:T}$ can be automatically checked by extracting and canonicalizing its final answer.
Concretely, let $m_x(y_{1:T})$ denote the answer extracted from the completion $y_{1:T}$ for prompt $x$, represented in a canonical form so that symbolically equivalent answers are identified; if no valid answer can be extracted, we set $m_x(y_{1:T})=\bot$.
A standard reward may be the correctness indicator $r_{\mathrm{accuracy}}:\mathcal{X}\times\mathbb{V}^{T}\to\{0,1\}$,
\[
r_{\mathrm{accuracy}}(x,y_{1:T})
:=
\mathbf{1}\!\left\{m_x(y_{1:T})=m_{\star}\right\},
\]

\paragraph{KL-regularized optimum.}
For regularization strength $\tau>0$, consider the problem
\begin{equation}\label{eq:kl-reg}
  \max_{q}\; \mathbb{E}_{x\sim\rho}\!\left[\mathbb{E}_{Y\sim q(\cdot\mid x)}[r(x,Y)] \;-\; \tau\,\operatorname{KL}\!\bigl(q(\cdot\mid x)\,\|\,p_0(\cdot\mid x)\bigr)\right],
\end{equation}
where the maximum is over all families of prompt-conditional trajectory laws $q(\cdot\mid x)$ on~$\mathbb{V}^T$.
Because the integrand decomposes into independent per-prompt terms, the optimum is attained by solving each prompt separately.
By the Gibbs variational principle~\citep{cover2006elements}, the unique per-prompt maximizer is (see Appendix~\ref{app:proof-qstar} for a self-contained proof; see also \citet{ziegler2019fine,korbak2022rl,rafailov2023direct,zhou2025qsharp} for related derivations in the RLHF setting)
\begin{equation}\label{eq:qstar}
  p^\star_\tau(y_{1:T}\mid x)
  = \frac{1}{Z(x)}\,p_0(y_{1:T}\mid x)\,e^{r(x,y_{1:T})/\tau},
  \qquad
  Z(x) := \mathbb{E}_{Y\sim p_0(\cdot\mid x)}\!\left[e^{r(x,Y)/\tau}\right].
\end{equation}

\paragraph{Per-step decomposition.}
The trajectory-level KL decomposes into a sum of per-step divergences: for any two autoregressive policies $\pi,\pi'$,
\[
  \operatorname{KL}\!\bigl(p_\pi(\cdot\mid x)\,\|\,p_{\pi'}(\cdot\mid x)\bigr)
  = \mathbb{E}_{Y\sim p_\pi}\!\left[\sum_{t=1}^{T}\operatorname{KL}\!\bigl(\pi(\cdot\mid x,y_{1:t-1})\,\|\,\pi'(\cdot\mid x,y_{1:t-1})\bigr)\right],
\]
by the chain rule for relative entropy~\citep{cover2006elements}.
Applying this to the per-prompt integrand in Equation~\eqref{eq:kl-reg}, for each prompt~$x$ the objective decomposes as
\[
  \mathbb{E}_{Y\sim q(\cdot\mid x)}\!\left[r(x,Y)-\tau\sum_{t=1}^{T}\operatorname{KL}\!\bigl(q(\cdot\mid x,y_{1:t-1})\,\|\,\pi_0(\cdot\mid x,y_{1:t-1})\bigr)\right],
\]
so optimizing over prompt-conditional trajectory laws is equivalent to optimizing over autoregressive policies.
Define the \emph{soft value}
\begin{equation}\label{eq:soft-value}
  Z(x,y_{1:t-1},y)
  := \mathbb{E}_{y_{t+1:T}\sim p_0(\cdot\mid x,y_{1:t-1},y)}\!\left[e^{r(x,y_{1:T}^{(t,y)})/\tau}\right],
\end{equation}
where \(y_{1:T}^{(t,y)}:=(y_{1:t-1},y,y_{t+1:T})\).\\

This is the expected exponentiated reward when token~$y$ is chosen at step~$t$ given prompt~$x$ and prefix~$y_{1:t-1}$, and the remaining tokens are rolled out under~$p_0$.
Marginalizing Equation~\eqref{eq:qstar} over future tokens $y_{t+1:T}$ and dividing by the prefix marginal yields the per-step conditionals of~$p^\star_\tau$ (see Appendix~\ref{app:proof-pi-star}):
\begin{equation}\label{eq:pi-star}
  \pi^\star_\tau(y_t\mid x,y_{1:t-1})
  = \frac{\pi_0(y_t\mid x,y_{1:t-1})\,Z(x,y_{1:t-1},y_t)}{\sum_{y'\in \mathbb{V}}\pi_0(y'\mid x,y_{1:t-1})\,Z(x,y_{1:t-1},y')}
  = \operatorname{softmax}\!\bigl(\ell(x,y_{1:t-1})+\log Z(x,y_{1:t-1},\cdot)\bigr)_{y_t}.
\end{equation}
Thus the optimal logit correction at step~$t$ is $\log Z(x,y_{1:t-1},\cdot)$, up to an additive constant absorbed by the softmax.

\begin{proposition}\label{prop:equivalence}
Sampling a trajectory $y_{1:T}$ from $p^\star_\tau$ in Equation~\eqref{eq:qstar} is equivalent to sampling autoregressively with $y_t\sim \pi^\star_\tau(\cdot\mid x,y_{1:t-1})$ in Equation~\eqref{eq:pi-star} for $t=1,\dots,T$.
That is, $p^\star_\tau(y_{1:T}\mid x)=\prod_{t=1}^T \pi^\star_\tau(y_t\mid x,y_{1:t-1})$.
\end{proposition}
\begin{proof}
Appendix~\ref{app:proof-pi-star} shows that marginalizing Equation~\eqref{eq:qstar} yields the conditionals in Equation~\eqref{eq:pi-star}.
The KL chain rule ensures that the trajectory-level objective decomposes into per-step terms, each uniquely maximized by Equation~\eqref{eq:pi-star}.
The product of these per-step conditionals recovers Equation~\eqref{eq:qstar}.
\end{proof}

\begin{corollary}[Exact factorization implies fixed-bias optimality]\label{cor:exact-factorization}
If $\log Z(x,y_{1:t-1},\cdot)=\delta+c_t(x,y_{1:t-1})\,\mathbf{1}$ for all $t\in\{1,\dots,T\}$, prompts $x$ in the support of $\rho$, and reachable prefixes $y_{1:t-1}$ (with fixed $\delta\in\mathbb{R}^{|\mathbb{V}|}$), then $p_\delta(\cdot\mid x)=p^\star_\tau(\cdot\mid x)$ for every such $x$.
\end{corollary}
\begin{proof}
Immediate from Equation~\eqref{eq:pi-star}, softmax shift invariance, and Proposition~\ref{prop:equivalence}; see Appendix~\ref{app:proof-exact-factorization}.
\end{proof}

\section{Method}
\label{sec:method}

We keep the base language model entirely frozen and steer its outputs by adding a single, context-independent logit-bias vector $\delta\in\mathbb{R}^{|\mathbb{V}|}$ to the logits at every decoding step.

We seek a logit-bias vector $\delta$ that approximates the unconstrained KL-regularized optimum~\eqref{eq:kl-reg}.
By Equation~\eqref{eq:pi-star}, the optimal logit correction at step~$t$ is $\log Z(x,y_{1:t-1},\cdot)$, up to an additive constant absorbed by the softmax.
This correction depends on the prompt and prefix, whereas a fixed logit-bias vector applies the same vector everywhere.
A single $\delta$ can recover the optimum exactly when $\log Z(x,y_{1:t-1},y)$ factorizes as $\delta(y) + c_t(x,y_{1:t-1})$ for all reachable prefixes---that is, when the token-to-token pattern of the ideal correction is the same everywhere, up to a prefix-dependent baseline.
Section~\ref{sec:theory} formalizes this as $\varepsilon$-approximate factorization and bounds the resulting 
reward gap.

\subsection{Population target and sampling procedure}\label{sec:closed-form}

Under the factorization assumption, the population-level bias is $\delta(y)=\log Z_{\mathrm{avg}}(y)+c$, where $c$ is a centering constant and
\begin{equation}\label{eq:Z-avg}
  Z_{\mathrm{avg}}(y)
  \;:=\; \mathbb{E}_{\substack{t\sim\mathrm{Unif}\{1,\dots,T\},\;x\sim\rho\\Y\sim p_0(\cdot\mid x,\,\mathrm{do}(y_t=y))}}\!\left[e^{r(x,Y)/\tau}\right]
  \;=\; \mathbb{E}_{\substack{t\sim\mathrm{Unif}\{1,\dots,T\},\;x\sim\rho\\y_{1:t-1}\sim p_0(\cdot\mid x)}}\!\left[Z(x,y_{1:t-1},y)\right],
\end{equation}
where $\mathrm{do}(y_t=y)$ denotes Pearl-style intervention at position $t$: $y_{1:t-1}\sim p_0(\cdot\mid x)$, $y_t=y$ is fixed, and $y_{t+1:T}\sim p_0(\cdot\mid x,y_{1:t-1},y)$; $Z$ is the soft value from Equation~\eqref{eq:soft-value}.
Thus $Z_{\mathrm{avg}}(y)$ is the expected exponentiated reward of a trajectory in which token $y$ is inserted at a uniformly random position and the prefix/suffix are drawn from the base model.

The data we use to estimate $Z_{\mathrm{avg}}(y)$ comes from on-policy rollouts: for each prompt, sample a full trajectory $y_{1:T}\sim p_0(\cdot\mid x)$ and draw $P\in\{1,\dots,T\}$ time-steps uniformly without replacement from $\{1,\dots,T\}$ (Algorithm~\ref{alg:data-gen}).
Concrete estimator for $Z_{\mathrm{avg}}(y)$ is given in Section~\ref{sec:param}.

\paragraph{On the choice of $P$.}
Samples within a trajectory share the prefix, suffix, and reward, so $P>1$ makes them non-i.i.d.\ across positions. Empirically, this dependence is negligible, while the $P$-fold data reuse---at fixed rollout budget, the dominant compute cost---materially cuts per-token variance.

\begin{algorithm}[t]
\caption{On-policy rollout data generation}
\label{alg:data-gen}
\begin{algorithmic}[1]
\Require prompts \(\{x^{(n)}\}_{n=1}^N\), base model \(\pi_0\), horizon \(T\), reward oracle \(r(\cdot,\cdot)\)
\Require number of sampled positions per trajectory \(P\in\{1,\dots,T\}\)
\For{\(n=1,\dots,N\)}
    \State Generate trajectory \(y_{1:T}^{(n)} \sim p_0(\cdot \mid x^{(n)})\) and query reward \(r^{(n)} := r\!\left(x^{(n)},y_{1:T}^{(n)}\right)\)
    \State Sample positions \(\{t^{(n,p)}\}_{p=1}^P\) uniformly \emph{without replacement} from \(\{1,\ldots,T\}\)
    \State For \(p=1,\dots,P\), set \(x^{(n,p)}:=x^{(n)}\), \(y_{1:T}^{(n,p)}:=y^{(n)}_{1:T}\), and \(r^{(n,p)}:=r^{(n)}\)
\EndFor
\State \Return \(\mathcal{D}=\bigl\{\bigl(x^{(n,p)},\,y_{1:T}^{(n,p)},\,t^{(n,p)},\,r^{(n,p)}\bigr)\bigr\}_{n,p=1}^{N,P}\)
\end{algorithmic}
\end{algorithm}

\subsection{Logit-bias estimation}\label{sec:param}

Given the on-policy dataset $\mathcal{D}$ generated by Algorithm~\ref{alg:data-gen}, we estimate 
$Z_{\mathrm{avg}}(y)$ using the inverse propensity scoring (IPS) estimator
\begin{equation}\label{eq:Z-hat}
  \widehat{Z}(y)
  :=
  \frac{1}{NP}
  \sum_{(x,\,y_{1:T},\,t,\,r)\,\in\,\mathcal{D}}
  \frac{
    \mathbf{1}\!\left\{y_t=y\right\}\,
    \exp\!\left(r/\tau\right)
  }{
    \pi_0\!\left(y \mid x,\,y_{1:t-1}\right)
  } .
\end{equation}

This estimator is unbiased:
$\mathbb{E}[\widehat{Z}(y)]=Z_{\mathrm{avg}}(y)$
(Proposition~\ref{prop:ips-unbiased}, proved in Appendix~\ref{app:finite-sample-est}).

In practice, we convert $\widehat Z$ into a logit-bias vector using additive smoothing: 
$\widehat{\delta}(y)
  =
  \log\!\bigl(\alpha+\widehat Z(y)\bigr)-c(\alpha),
  \qquad
  c(\alpha):=\frac{1}{|\mathbb V|}\sum_{y'\in\mathbb V}\log\!\bigl(\alpha+\widehat Z(y')\bigr).$
The smoothing constant $\alpha>0$ acts as a pseudocount~\citep{lidstone1920bayeslaplace,chen1999smoothing}: 
it shrinks weak evidence toward a common baseline, prevents $\log 0$, and bounds the local sensitivity of 
the log transform by $1/\alpha$. We choose $\alpha$ by validation performance.

\section{Theoretical Analysis}
\label{sec:theory}

We establish that the closed-form estimator from Section~\ref{sec:method} is near-optimal when the soft value $\log Z(x,y_{1:t-1},y)$ approximately factorizes.
Full proofs appear in Appendix~\ref{app:proofs}.

\subsection{Approximate factorization}

As in Section~\ref{sec:method}, the ideal logit correction at step~$t$ is $\log Z(x,y_{1:t-1},\cdot)$ (up to a constant), which a fixed logit-bias vector $\delta$ represents exactly under the factorization assumption $\log Z(x,y_{1:t-1},y) = \delta(y) + c_t(x,y_{1:t-1})$ for all reachable prefixes; we now relax this to an approximate version.

\begin{definition}[$\varepsilon$-Approximate Factorization]\label{def:approx-fact}
A logit-bias vector $\widehat{\delta}\in\mathbb{R}^{|\mathbb{V}|}$ satisfies $\varepsilon$-approximate factorization if, for every $t\le T$, every prompt~$x$ in the support of~$\rho$, and every prefix~$y_{1:t-1}$ reachable under~$p^\star_\tau$ at time~$t$, there exists $c_t(x,y_{1:t-1})\in\mathbb{R}$ with
\[
  \bigl\|\log Z(x,y_{1:t-1},\cdot)-(\widehat{\delta}+c_t(x,y_{1:t-1})\mathbf{1})\bigr\|_2\le \varepsilon.
\]
\end{definition}

When $\varepsilon=0$ this is exact factorization; the closed-form estimator in Section~\ref{sec:closed-form} targets exactly this regime.

\subsection{Reward gap bound}

\begin{theorem}[Reward gap]\label{thm:strict}
Assume $r:\mathcal{X}\times\mathbb{V}^T\to[0,R]$ with $R>0$ and $\tau>0$.
Let
\[
  V_0 \;:=\; \mathbb{E}_{x\sim\rho}\!\left[\operatorname{Var}_{Y\sim p_0(\cdot\mid x)}\!\bigl(r(x,Y)\bigr)\right]
\]
denote the expected reward variance under the base model.
If $\widehat{\delta}\in\mathbb{R}^{|\mathbb{V}|}$ satisfies $\varepsilon$-approximate factorization assumption (Definition~\ref{def:approx-fact}) for every prompt in the support of~$\rho$, then
\begin{equation}\label{eq:reward-gap}
  J(p_{\widehat{\delta}})-J(p_0)\;\ge\;
  \underbrace{\frac{V_0}{R}\Bigl(1-e^{-R/\tau}\Bigr)}_{\text{improvement of } p^\star_\tau \text{ over } p_0}
  \;-\;
  \underbrace{R\,\varepsilon\sqrt{\frac{T}{8}}}_{\text{cost of fixed-bias approximation}}.
\end{equation}
\end{theorem}

\begin{proof}
See Appendix~\ref{app:proof-reward-gap}.
\end{proof}

The first term is a lower bound on $J(p^\star_\tau)-J(p_0)$, the reward gain of the unconstrained KL-regularized optimum over the base model; it grows with the expected reward variance~$V_0$ and is strictly positive whenever $V_0>0$, since $R>0$ implies $e^{-R/\tau}<1$.
The second term is the price of restricting to a context-independent logit-bias vector; it grows with the approximation error~$\varepsilon$ and the horizon~$T$.
Logit-bias steering strictly improves over the base model whenever the expected reward variance is large enough to dominate the approximation error.
The bound applies to \emph{any} $\widehat{\delta}$ satisfying Definition~\ref{def:approx-fact}.

\paragraph{Reward normalization and the role of~$R$.}
Both terms of the bound in Theorem~\ref{thm:strict} depend on the reward range~$R$: the gain scales as $V_0/R$ while the error scales as $R\,\varepsilon\sqrt{T/8}$.
Keeping $R$ small is therefore essential for a tight guarantee.
Adding a constant to~$r$ shifts neither~$p^\star_\tau$ nor the reward gap, so the binding quantity is the \emph{range} of~$r$, not its location.
For pure binary accuracy, $R=1$.
Adding an unnormalized length penalty $-L(y)$ would inflate the range to $L_{\max}+1=2{,}049$, degrading the bound by a factor exceeding $10^3$.
Normalizing by $L_{\max}$ yields $r(x,y_{1:T})=\mathbf{1}\{m_x(y_{1:T})=m_{\star}\}-L(y_{1:T})/L_{\max}$ with range~$2$, preserving a bound of the same order as the binary case.
This motivates the length normalization used in the experiments.

\section{Experiments}\label{sec:experiments}

\subsection{Experimental setup}\label{sec:experimental-setup}
{\setlength{\parskip}{0pt}
\makeatletter
\renewcommand\paragraph{\@startsection{paragraph}{4}{\z@}%
  {0.25ex \@plus 0.1ex \@minus 0.1ex}%
  {-0.8em}%
  {\normalfont\normalsize\bfseries}}
\makeatother
\paragraph{Models.}
Accuracy uses Meta-Llama-3.1-8B-Instruct~\citep{llama3} as the base model; the compression objective adds Gemma-2-9B-IT~\citep{gemma2} and Qwen3-4B~\citep{qwen3}.

\paragraph{Tasks and data.}
We evaluate on MATH, GSM8K, and GPQA-main~\citep{hendrycks2021math,hf_hendrycks_math,cobbe2021gsm8k,hf_gsm8k,rein2023gpqa,hf_gpqa} with the default lm-evaluation-harness prompts~\citep{biderman2024lm_eval,lm_eval_harness}. MATH and GSM8K use the standard HuggingFace train/test splits with validation sampled from train; GPQA-main is split 80/10/10. For GPQA-main we additionally shuffle answer choices and require explanation-first, answer-last formatting, so sampled positions for bias estimation are not tied to a fixed choice index (Appendix~\ref{app:exp-details}).
\paragraph{Rewards.}
For accuracy, the logit-bias method is learned from the log-indicator variant of the RLVR accuracy reward $r_{\mathrm{accuracy}}$ (Section~\ref{sec:preliminaries}),
\[
r_{\text{log-accuracy}}(x,y_{1:T})
:=
\log \mathbf{1}\!\left\{m_x(y_{1:T})=m_{\star}\right\}
=
\begin{cases}
0, & m_x(y_{1:T})=m_{\star},\\
-\infty, & \text{otherwise}.
\end{cases}
\]
With the convention that $\log 0=-\infty$, this encodes the same correctness check as $r_{\mathrm{accuracy}}$ in the log-reward form expected by the KL-tilted objective~\eqref{eq:kl-reg}: since $e^{r_{\text{log-accuracy}}/\tau}=\mathbf{1}\{m_x(y_{1:T})=m_{\star}\}$ for every $\tau>0$, the estimator $\widehat{Z}$ reduces to an accuracy-weighted count and $\tau$ becomes inert, eliminating the need to tune it.
Answers are extracted and canonicalized by the lm-evaluation-harness verifier; the reward is binary with no partial credit, and an unextractable answer ($m_x(y_{1:T})=\bot$) counts as incorrect. 

We also evaluate the compression objective as a second reward-driven adaptation target, useful both practically and diagnostically. Practically, shorter completions reduce serving cost and latency, and are often desirable when users need concise answers rather than full reasoning traces. Diagnostically, it stress-tests whether the fixed logit-bias vector $\delta$ can learn reusable, context-independent answer-finalization and stopping patterns rather than new problem-specific reasoning.
For the compression reward, the logit-bias method uses
\begin{equation}\label{eq:r-length}
r_{\mathrm{length}}(y)=\ln\!\left(T/\operatorname{length}(y)\right).
\end{equation}
Rollout-generation and LoRA settings 
are reported in Appendix~\ref{app:exp-details}.

\paragraph{Baselines.}
We compare against three baselines: 
(i) the base model, 
(ii) a 32-cluster tied logit-bias baseline that groups tokens with similar surface strings via $k$-means and shares one bias per cluster (Appendix~\ref{app:exp-details}), and 
(iii) a LoRA-GRPO baseline~\citep{hu2021lora,shao2024deepseekmath} that fine-tunes the same base model with 
checkpoint selection on validation. LoRA-GRPO uses the binary accuracy reward
$r^{\mathrm{GRPO}}_{\mathrm{accuracy}}(x,y_{1:T})
=
\mathbf{1}\{m_x(y_{1:T})=m_{\star}\}$
in the accuracy setting, and the accuracy-gated compression reward
$r^{\mathrm{GRPO}}_{\mathrm{length}}(x,y_{1:T})
=
\mathbf{1}\{m_x(y_{1:T})=m_{\star}\}\cdot
T/\operatorname{length}(y_{1:T})$
in the compression experiment. Construction and training details are in 
Appendix~\ref{app:exp-details}.
\paragraph{Hyperparameters and evaluation.}
For accuracy runs, $\tau$ is inert under $r_{\text{log-accuracy}}$, so we tune only $\alpha\in[0.005,0.11]$ on validation. In the compression experiment, we tune both $\alpha\in[0.005,0.11]$ and $\tau\in[0.5,1.5]$. Accuracy runs select the configuration maximizing validation exact-match accuracy; compression experiment runs treat the two as a Pareto tradeoff and pick the shortest completions that preserve validation accuracy relative to the base model. We apply learned logit-bias vectors via the vLLM logits-processor API under greedy decoding, reporting exact-match accuracy and completion length in tokens. Confidence intervals are computed using percentile bootstrap resampling; see Appendix~\ref{app:bootstrap}.

}

\subsection{Main results}

\begin{table}[t]
  \centering
  \caption{Exact-match accuracy (\%) on MATH ($n{=}5000$), GSM8K ($n{=}1319$), and GPQA ($n{=}45$) under the accuracy-reward setup.
  Each cell is empirical accuracy with a percentile bootstrap 95\% confidence interval (10,000 resamples), shown as $\pm$ percentage-point offsets. Higher is better.}
  \label{tab:accuracy_results}
  \begin{tabular}{llccc}
  \toprule
  Method                              & \shortstack{Trained\\Params} & MATH & GSM8K & GPQA  \\
  \midrule
  Base                          & --                  & $30.30 \pm 1.26$  & $33.89 \pm 2.50$ & $22.20 \pm 13.33$        \\
  \midrule
  Logit bias (32 clusters)                  & 32            & $30.12 \pm 1.26$  & $33.30 \pm 2.58$  & $24.40 \pm 13.33$       \\
  Logit bias (full-vocabulary)              & $128\text{k}$ & $33.14 \pm 1.30$  & $34.80 \pm 2.50$ & $28.90 \pm 13.33$       \\
  LoRA-GRPO                                & $\sim 10^7$    & $45.24 \pm 1.40$  & $76.72 \pm 2.27$  & $31.11 \pm 13.33$         \\
  \bottomrule
  \end{tabular}
\end{table}

\begin{table}[t]
  \centering
  \caption{Mean completion length in tokens on MATH for Gemma-9B, Llama-8B, and Qwen3-4B under Base, full-vocabulary logit bias, and LoRA-GRPO. Each cell reports the empirical mean with a percentile bootstrap 95\% confidence interval (10,000 resamples) from $n=5000$ examples, shown as $\pm$ offsets from the mean. Lower is better.}
  \label{tab:compression_results}
  \begin{tabular}{llccc}
  \toprule
  Method                              & \shortstack{Trained\\Params} &  gemma-9B &  Llama-8B & Qwen3-4B  \\
  \midrule
  Base                          &      --                         &  $315.78 \pm 4.72$ &  $941.02 \pm 6.98$ & $1018.12 \pm 1.57$        \\
  \midrule
  Logit bias                    & $150\text{k}$–$256\text{k}$           &  $299.32 \pm 6.11$ &  $864.24 \pm 9.02$ & $1017.61 \pm 1.59$        \\
  LoRA-GRPO                                & $\sim 10^7$                & $5.53 \pm 0.15$     & $466.40 \pm 6.63$  & $275.94 \pm 6.23$        \\
  \bottomrule
  \end{tabular}
\end{table}

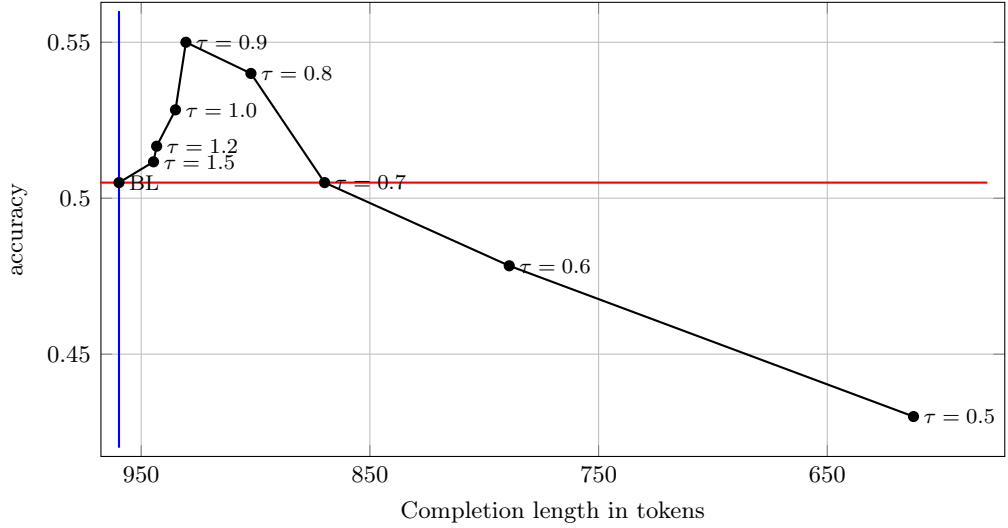
\begin{figure}[t]
  \centering
  \begin{tikzpicture}
    \begin{axis}[
      width=0.82\linewidth,
      height=0.46\linewidth,
      xlabel={Completion length in tokens},
      ylabel={accuracy},
      xmin=580, xmax=960,
      x dir=reverse,
      xtick={950,850,750,650},
      ymin=0.42, ymax=0.56,
      enlarge x limits=0.02,
      enlarge y limits=0.02,
      grid=major,
      tick label style={font=\footnotesize},
      label style={font=\footnotesize}
    ]
      \addplot[
        color=red,
        thick,
        domain=580:980
      ] {0.5050};
      \addplot[
        color=blue,
        thick
      ] coordinates {(959.71,0.42) (959.71,0.56)};

      \addplot[
        thick,
        mark=*,
        mark size=1.7pt,
        color=black
      ] coordinates {
        (959.71,0.5050)
        (944.606667,0.511667)
        (943.25,0.516667)
        (935.001667,0.528333)
        (930.446667,0.55)
        (902.011667,0.54)
        (869.806667,0.505)
        (789.068333,0.478333)
        (612.301667,0.43)
      };
      \node[font=\scriptsize, anchor=west] at (axis cs:959.71,0.5050) {BL};
      \node[font=\scriptsize, anchor=west] at (axis cs:944.606667,0.511667) {$\tau=1.5$};
      \node[font=\scriptsize, anchor=west] at (axis cs:943.25,0.516667) {$\tau=1.2$};
      \node[font=\scriptsize, anchor=west] at (axis cs:935.001667,0.528333) {$\tau=1.0$};
      \node[font=\scriptsize, anchor=west] at (axis cs:930.446667,0.55) {$\tau=0.9$};
      \node[font=\scriptsize, anchor=west] at (axis cs:902.011667,0.54) {$\tau=0.8$};
      \node[font=\scriptsize, anchor=west] at (axis cs:869.806667,0.505) {$\tau=0.7$};
      \node[font=\scriptsize, anchor=west] at (axis cs:789.068333,0.478333) {$\tau=0.6$};
      \node[font=\scriptsize, anchor=west] at (axis cs:612.301667,0.43) {$\tau=0.5$};
    \end{axis}
  \end{tikzpicture}
  \caption{Validation accuracy--completion-length Pareto frontier across $\tau$ on MATH (Llama-3.1-8B-Instruct). Higher accuracy and shorter completions are preferred; horizontal and vertical lines mark the base model. Lower $\tau$ yields stronger compression under the compression reward, exposing the model-selection tradeoff.}
  \label{fig:compressive-pareto}
\end{figure}

\paragraph{Logit-bias statistics.}
Table~\ref{tab:delta-stats} summarizes the learned logit-bias vector $\widehat{\delta}$ after centering to zero mean.
As a sparsity descriptor we report the \emph{selected-coordinate fraction}: the fraction of vocabulary coordinates whose absolute bias exceeds the \emph{effect-size gate} $|\widehat{\delta}(a)| \ge 5\,\operatorname{median}|\widehat{\delta}|$.
The qualitative analysis (Section~\ref{sec:discussion}, Appendices~\ref{app:accuracy-delta-analysis} and~\ref{app:compression-delta-analysis}) instead reads each vector's realized decode-time effect through the per-token \emph{intervention score} $S$ (Appendix~\ref{app:intervention-score}), splitting the visited vocabulary by the sign of $S$ rather than thresholding the raw magnitude of $\widehat{\delta}$.

\begin{table}[t]
  \centering
  \caption{Summary statistics of the learned logit-bias vector $\widehat{\delta}$ after centering to zero mean. The selected-coordinate fraction counts coordinates satisfying the effect-size gate $|\widehat{\delta}(a)|\ge 5\,\operatorname{median}|\widehat{\delta}|$.}
  \label{tab:delta-stats}
  {\small
  \begin{tabular}{lccccc}
  \toprule
  Reward & Sampled fraction & $\sigma(\delta)$ & $\min$ & $\max$ & \shortstack{Selected-coordinate\\fraction} \\
  \midrule
  Accuracy    & \multirow{2}{*}{$28.5\%$} & $7.67 \times 10^{-3}$ & $-2.50 \times 10^{-4}$ & $1.574$ & $1.30\%$ \\
  Compression &                           & $4.56 \times 10^{-2}$ & $-2.37 \times 10^{-3}$ & $4.713$ & $1.36\%$ \\
  \bottomrule
  \end{tabular}}
\end{table}
Table~\ref{tab:delta-stats} shows that only a small selected-coordinate fraction passes the effect-size gate, suggesting sparse token-level nudges toward answer boundaries, mathematical formatting, and compact finalization rather than broad retuning.

{\setlength{\parskip}{0pt}
\makeatletter
\renewcommand\paragraph{\@startsection{paragraph}{4}{\z@}%
  {0.25ex \@plus 0.1ex \@minus 0.1ex}%
  {-0.8em}%
  {\normalfont\normalsize\bfseries}}
\makeatother
\paragraph{Accuracy results.}
Table~\ref{tab:accuracy_results} evaluates logit-bias adaptation when the reward is instantiated as answer accuracy. Full-vocabulary logit bias gives small but consistent gains over the base model on all benchmarks, most clearly on MATH ($30.30\!\to\!33.14$) and GPQA ($22.20\!\to\!28.90$). In contrast, the 32-cluster baseline stays close to the base model, suggesting that the useful signal is token-specific rather than a coarse frequency shift. LoRA-GRPO remains stronger by adapting context-dependent representations across the network, whereas logit bias is limited to a fixed vocabulary-level shift.
\paragraph{Compression results.}
Table~\ref{tab:compression_results} evaluates logit-bias adaptation under the compression reward, constrained to preserve answer quality: we sweep $\tau$ on validation and keep short completions only among settings that preserve base-model accuracy. The effect is modest but useful---the fixed logit-bias vector $\delta$ shortens Gemma and Llama outputs, with little effect on Qwen3-4B---and Figure~\ref{fig:compressive-pareto} shows the accuracy--completion-length frontier, where lowering $\tau$ first shortens completions without hurting accuracy and eventually causes early stopping and accuracy loss. LoRA-GRPO compresses substantially more by changing context-dependent reasoning and stopping behavior.
}

\section{Discussion}
\label{sec:discussion}

{\setlength{\parskip}{0pt}
\makeatletter
\renewcommand\paragraph{\@startsection{paragraph}{4}{\z@}%
  {0.25ex \@plus 0.1ex \@minus 0.1ex}%
  {-0.8em}%
  {\normalfont\normalsize\bfseries}}
\makeatother

\paragraph{Summary of findings.}
Across both adaptation targets, the fixed logit-bias vector $\delta$ provides lightweight black-box steering of the base model: under the accuracy reward it yields small but consistent gains (Table~\ref{tab:accuracy_results}), and under the compression reward it traces an accuracy--completion-length tradeoff rather than unconstrained reduction of completion length in tokens (Table~\ref{tab:compression_results}, Figure~\ref{fig:compressive-pareto}). Because the same token-level correction applies at every prefix, the bias steers reusable formatting, finalization, and stopping patterns but cannot adapt context-dependent reasoning---the gap where LoRA-GRPO remains stronger (Section~\ref{sec:limitations}).

\paragraph{Structure of the learned biases.}
To read each vector's structure we summarize its decode-time effect by the per-token \emph{intervention score} $S(y)$: the average log-probability shift the bias induces on token $y$ at the positions where the fixed-bias policy generates it, so $S(y)>0$ ($<0$) marks tokens it makes more (less) likely (Appendix~\ref{app:intervention-score}). Splitting the visited vocabulary by the sign of $S$ and grouping each side into semantic families gives the analyses in Appendices~\ref{app:accuracy-delta-analysis} and~\ref{app:compression-delta-analysis}, with complete suppressed-side listings provided in the accompanying code repository. At a high level, the accuracy bias promotes the surface and discourse regularities of correct solutions (layout and math delimiters, solution openers, digits, and a decisive stop string), while its suppressed side acts as an anti-drift prior against question-restatement, lesson/tutorial mode, scraped page furniture, and invented problem-specific detail---favoring a clean, staged path to the answer rather than steering the mathematics itself. The compression bias is markedly \emph{two-sided}: a small, sharp promoted set of stop, boundary, and compact-format tokens against a broad suppressed set of explanation, derivation, and continuation tokens, so its main risk is \emph{over-compression} (Section~\ref{sec:limitations}).

\paragraph{Decode-time mechanism and fixed-bias projection.}
These patterns reflect a single decode-time mechanism: the method projects a reward-tilted target policy onto a highly restricted class---policies that add the same vector to the vocabulary logits at every prefix---which is effective when the ideal correction $\log Z(x,y_{<t},\cdot)$ is approximately token-specific and stable across contexts, as it is for formatting, termination, answer-surface regularities, and verbosity control. Applied at every step, these local nudges accumulate into trajectory-level changes that extract useful task signal from black-box rollouts, explaining why the bias improves over the base model (Table~\ref{tab:accuracy_results}) and reduces completion length in tokens (Table~\ref{tab:compression_results}); Figure~\ref{fig:compressive-pareto} shows the same mechanism at the hyperparameter level, with stronger compression moving along the accuracy--completion-length frontier rather than minimizing length unconditionally.

}

\section{Limitations}
\label{sec:limitations}

{\setlength{\parskip}{0pt}
\makeatletter
\renewcommand\paragraph{\@startsection{paragraph}{4}{\z@}%
  {0.25ex \@plus 0.1ex \@minus 0.1ex}%
  {-0.8em}%
  {\normalfont\normalsize\bfseries}}
\makeatother

\paragraph{Expressivity and access--performance tradeoff.}
This restriction is also the central limitation. The fixed logit-bias vector $\delta$ cannot rewrite reasoning, encode conditional rules, or learn prompt-specific corrections (e.g., which theorem or calculation is needed, which value is correct, or which steps are safe to omit); it can promote formatting, digits, and finalization markers and push down drift and continuation tokens, but it cannot supply the correct reasoning content. Because the same correction applies at every prefix, this context-blindness also makes compression prone to \emph{over-compression}, suppressing useful verification or necessary derivation along with genuine verbosity. This is the gap LoRA-GRPO closes in Tables~\ref{tab:accuracy_results} and~\ref{tab:compression_results}: weight-space adaptation changes internal, context-dependent behavior, whereas the fixed logit-bias vector $\delta$ applies the same token-level correction at every prefix. Learned logit bias is thus best viewed not as a replacement for fine-tuning, but as a lightweight black-box primitive when gradients, weights, and provider-side training are unavailable.

\paragraph{Estimation limits.}
Learned logit bias is also constrained by estimation. The IPS estimator can be high-variance for rare tokens, additive smoothing introduces shrinkage, and finite rollout budgets leave many vocabulary coordinates weakly supported. As a result, some promoted coordinates are spurious rare-token or dataset artifacts rather than meaningful steering directions---our intervention-score analysis flags several such tokens (Appendices~\ref{app:accuracy-delta-analysis} and~\ref{app:compression-delta-analysis}; full suppressed-side listings in the accompanying code repository)---though they are often inert in practice, since the base model assigns them negligible probability in mathematical contexts. This highlights the need for better smoothing and finite-sample analysis, especially for objectives where rollouts give weak evidence about which local token choices would have improved the final reward.

}

\section{Conclusion}
\label{sec:conclusion}

We adapted frozen language models under black-box constraints by learning a single, context-independent logit-bias vector $\delta$ from rollouts, rewards, and token probabilities, applied at decoding.
From a KL-regularized objective, we derived the prefix-dependent optimal correction, characterized when a fixed logit-bias vector approximates it, and obtained a closed-form IPS estimator with a reward-gap bound.

Empirically, full-vocabulary logit bias improves over the base model on accuracy benchmarks without gradient access or weight updates, and supports the compression objective by learning reusable answer-finalization and stop-boundary patterns, though LoRA-GRPO remains substantially stronger when weight-space training is available.
Overall, learned logit bias is a lightweight adaptation primitive between prompt engineering and fine-tuning; future work should strengthen finite-sample theory, improve rare-token smoothing, extend to context-dependent biases, and broaden safety, privacy, fairness, and user-facing evaluation.

\section*{Acknowledgments}
This work is partially supported by the Israeli Science Foundation (ISF) grant no 4032/25.

\bibliographystyle{plainnat}
\bibliography{citations}


\appendix
\renewcommand{\thefigure}{\Alph{section}.\arabic{figure}}
\renewcommand{\thetable}{\Alph{section}.\arabic{table}}
\renewcommand{\theHfigure}{\Alph{section}.\arabic{figure}}
\renewcommand{\theHtable}{\Alph{section}.\arabic{table}}
\pretocmd{\section}{\setcounter{figure}{0}\setcounter{table}{0}}{}{}

\etocdepthtag.toc{appendix}
\begingroup
\etocsettagdepth{mainbody}{none}
\etocsettagdepth{appendix}{subsection}
\etocsettocstyle{\section*{Appendix Contents}}{}
\tableofcontents
\endgroup

\section{Intervention score, token selection, and grouping}
\label{app:promoted-cluster-calculation}

Appendices~\ref{app:accuracy-delta-analysis} and~\ref{app:compression-delta-analysis} interpret a
learned logit-bias vector $\widehat{\delta}\in\mathbb{R}^{|\mathbb{V}|}$ by (i)~summarizing its
decode-time effect as a per-token \emph{intervention score} $S$,
(ii)~splitting the visited tokens into promoted and suppressed sets by the sign of
$S$, and (iii)~grouping those into semantic families.  Complete suppressed-side listings are provided in
the code repository under \texttt{results/}.

\subsection{Intervention score as realized log-probability shift}
\label{app:intervention-score}

A raw $\widehat{\delta}(y)$ lives in logit space, not directly in probability space: after softmax
normalization, the actual change in a token's probability depends on the prefix at which it is decoded.
We therefore summarize the decode-time effect of the learned bias by measuring, at each generated
position, the log-probability shift it induces on the realized next token.
For a token $y\in\mathbb{V}$, averaging this shift over the positions where the fixed-bias policy
generates $y$ gives the occurrence-weighted \emph{intervention score}
\begin{equation}\label{eq:intervention-score}
  S(y) \;:=\; \mathbb{E}_{\substack{x\sim\rho,\;t\sim\mathrm{Unif}\{1,\dots,T\}\\ y_{1:T}\sim p_{\widehat{\delta}}(\cdot\mid x)}}\!\left[\,\log\frac{\pi_{\widehat{\delta}}(y\mid x,y_{1:t-1})}{\pi_0(y\mid x,y_{1:t-1})}\;\middle|\;y_t=y\,\right].
\end{equation}
The expectation is occurrence-weighted: it conditions on the positions where the fixed-bias policy
realizes $y$, with the prefix $y_{1:t-1}$ and the position $t$ sampled
uniformly.  Since the log-ratio depends only on $x$ and the prefix $y_{1:t-1}$, the suffix $y_{t+1:T}$ is
irrelevant.  Because both probabilities are post-softmax, $S(y)$ reflects the actual probability change
induced by the bias at decode time, not the raw coordinate $\widehat{\delta}(y)$: $S(y)>0$ promotes the
realized token and $S(y)<0$ suppresses it.
This log-ratio is the same local quantity that appears in the KL-regularized derivation, but here we use
it only as an interpretability score.

\paragraph{Estimator.}
We estimate $S$ on held-out \emph{on-policy} rollouts generated by the fixed-bias policy
$p_{\widehat{\delta}}$.  Given such a rollout set with recorded token ids, we use every position in
each trajectory. At each position $(n,t)$ we run the base model on the
prompt and the fixed-bias-policy rollout prefix to obtain its logits, form $\pi_0$ and
$\pi_{\widehat{\delta}}$ from them, and record the shift on the realized token,
\[
  \Delta_{t}^{(n)}
  \;:=\;
  \log\pi_{\widehat{\delta}}\!\left(y_{t}^{(n)}\mid x^{(n)},y_{1:t-1}^{(n)}\right)
  -\log\pi_0\!\left(y_{t}^{(n)}\mid x^{(n)},y_{1:t-1}^{(n)}\right).
\]
The score of a token is the mean of these shifts over the positions at which it is realized, 
\[
  \widehat{S}(y)
  \;=\;
  \frac{1}{|\mathcal{I}_y|}\sum_{(n,t)\in\mathcal{I}_y}\Delta_{t}^{(n)},
  \qquad
  \mathcal{I}_y:=\bigl\{(n,t):y_{t}^{(n)}=y\bigr\},
\]
with the convention $\widehat{S}(y):=0$ for tokens that never appear at any position in the rollouts
($\mathcal{I}_y=\emptyset$).

\subsection{Splitting tokens by sign}
\label{app:token-selection}
We apply no magnitude threshold: every visited token is assigned to one side by the sign of its score, 
\[
    \mathcal{R}^{+}=\bigl\{\,y\in\mathbb{V}:\ S(y)>0\,\bigr\},
    \qquad
    \mathcal{R}^{-}=\bigl\{\,y\in\mathbb{V}:\ S(y)<0\,\bigr\}.
\]
The \emph{promoted} tokens $\mathcal{R}^{+}$ are those the bias makes more likely at decode time and the
\emph{suppressed} tokens $\mathcal{R}^{-}$ those it makes less likely; the remaining tokens, never
visited under $p_{\widehat{\delta}}$ and hence carrying $S(y)=0$, are left ungrouped.

\subsection{Grouping selected tokens into families}
\label{app:embedding-clusters}
To summarize each side ($\mathcal{R}^{+}$ and $\mathcal{R}^{-}$) we group its tokens into
semantic families.  This grouping is a purely interpretive aid, not part of the method, and the
families carry no weight of their own; we refer to them by index in the reward-specific appendices
below.  The promoted set is small and dominated by a few recurring surface patterns---formatting,
connectives, digits, and stop tokens---so we group the promoted tokens manually by inspection.
The suppressed set is substantially larger and noisier, so we group it by AI-assisted manual curation,
with every suppressed token listed in the code repository under \texttt{results/}.

\section{Qualitative analysis of the accuracy logit-bias vector}
\label{app:accuracy-delta-analysis}
This appendix examines what the accuracy intervention score $S$ (Appendix~\ref{app:intervention-score})
encodes under the \emph{accuracy reward}.  Of the Llama vocabulary ($128{,}256$ tokens), $5{,}739$ are
visited in the rollouts---$137$ \emph{promoted} ($S>0$) and $5{,}602$ \emph{suppressed} ($S<0$)---and we
group each side into semantic families as in Appendix~\ref{app:promoted-cluster-calculation}; all
$5{,}602$ suppressed tokens with their scores are listed in the code repository under
\texttt{results/accuracy\_suppressed\_all\_families.tex}.

\begin{figure}[t]
  \centering
  \begin{tikzpicture}
    \begin{axis}[
      name=mainax,
      width=0.75\textwidth,
      height=0.52\textwidth,
      ymode=log,
      log origin=infty,
      ymin=0.5,
      xmin=-3.40, xmax=0.03,
      axis lines=left,
      xlabel={$S(y)$},
      ylabel={Number of tokens in bin},
      tick label style={font=\footnotesize},
      label style={font=\footnotesize},
      ybar,
      bar width=0.009,
    ]
      \addplot[fill=gray!45, draw=none, bar shift=0pt] table[
        x=bin_center,
        y expr={\thisrow{count}<1 ? 0.5 : \thisrow{count}},
      ] {data/accuracy_score_suppressed.tsv};
    \end{axis}
    \begin{axis}[
      at={(mainax.north west)}, anchor=north west,
      xshift=18pt, yshift=-5pt,
      width=0.42\textwidth, height=0.30\textwidth,
      ymode=log, log origin=infty, ymin=0.5,
      axis lines=left,
      xmin=-0.01, xmax=2.95,
      tick label style={font=\tiny},
      title={\scriptsize promoted side ($S(y)>0$)},
      title style={yshift=-4pt},
      xlabel={$S(y)>0$}, xlabel style={font=\tiny, yshift=3pt},
      ybar, bar width=0.02,
    ]
      \addplot[fill=blue!50, draw=none, bar shift=0pt] table[
        x=bin_center,
        y expr={\thisrow{count}<1 ? 0.5 : \thisrow{count}},
      ] {data/accuracy_score_promoted.tsv};
    \end{axis}
  \end{tikzpicture}
  \caption{Histogram of the intervention score $\widehat{S}(y)$ for the accuracy reward over the Llama
  vocabulary (visited tokens only; $122{,}517$ unvisited tokens with $\widehat{S}(y)=0$ are omitted).
  The $y$-axis is logarithmic. The main panel shows the \emph{suppressed} side ($\widehat{S}(y)<0$,
  $5{,}602$ tokens), which carries the bulk of the mass and a long tail reaching
  $\widehat{S}(y)\approx-3.36$; the inset zooms into the \emph{promoted} side ($\widehat{S}(y)>0$, $137$ tokens, up to $\widehat{S}(y)\approx2.84$).}
  \label{fig:accuracy-delta-histogram}
\end{figure}

The accuracy score is a \emph{two-sided correctness prior}: it promotes tokens that help format, scaffold,
and terminate correct solutions, and it suppresses tokens that usually start explanations, page drift,
or problem-specific detours.  The positive side is small and sharp while the negative side is broad
(Table~\ref{tab:accuracy-score-bands}), so the bias acts mainly by discouraging failure-mode
continuations while also boosting a few answer-surface and stop tokens.

\begin{table}[h]
  \centering
  \caption{Score bands of the accuracy intervention score $S(y)$.}
  \label{tab:accuracy-score-bands}
  {\small
  \begin{tabular}{lcl}
  \toprule
  Score band & Count & Interpretation \\
  \midrule
  $S\ge 1$              & $7$       & strongest promoted tokens \\
  $0.1\le S<1$          & $17$      & strong positive accuracy signal \\
  $0.01\le S<0.1$       & $39$      & moderate positive signal \\
  $0<S<0.01$            & $74$      & weak positive / noisy \\
  $-0.001\le S<0$       & $666$     & almost neutral negative \\
  $-0.01\le S<-0.001$   & $690$     & weak suppression \\
  $-0.1\le S<-0.01$     & $2{,}260$ & meaningful suppression \\
  $-1\le S<-0.1$        & $1{,}927$ & strong suppression \\
  $S<-1$                & $59$      & strongest negative shifts \\
  \bottomrule
  \end{tabular}}
\end{table}

\subsection{Promoted token families}
\label{app:accuracy-promoted-token-families}

The promoted side ($S(y)>0$) contains the $137$ tokens with a positive log-probability shift; we list
every one, grouped into families with a short gloss noting why each token plausibly co-occurs with
correct solutions.  Interpretation is collected in Appendix~\ref{app:accuracy-interpretation}.

Families are ordered by their mean intervention score $\bar{S}$ (the average of $S(y)$ over the family's tokens).

\begin{itemize}
    \item \textbf{Stop strings (2, $\bar{S}=1.66$).}
    \verb*!<|end_of_text|>! and \verb*| Problem|, which terminate the current completion.

    \item \textbf{Instruction / educational metadata (6, $\bar{S}=1.06$).}
    \verb*| Keywords| (label such as ``Problem Keywords''), \verb*|Basic| (tags like ``Basic Algebra''),
    \verb*|User| (AoPS/wiki \verb*|/User:| fragments), \verb*| Vocabulary| (``Key Vocabulary''),
    \verb*|Personal| (page/persona drift), and \verb*|Skill| (``Reasoning Skill'').

    \item \textbf{Text / document / page-drift concepts (2, $\bar{S}=0.51$).}
    \verb*|Ann| (name prefix, page/person drift) and \verb*| topics| (generic page language such as
    ``other topics'').

    \item \textbf{Actions / commands / ordering / problem-solving context (34, $\bar{S}=0.078$).}
    \verb*|Problem|, \verb*| Solutions|, \verb*| Since|, \verb*| number|, \verb*| logically|, \verb*| Answer|,
    \verb*| answer|, \verb*| Compute|, \verb*| find|, \verb*| write|, \verb*| use|, \verb*|use|, \verb*|align|,
    \verb*|edit|, \verb*|istinguish| (suffix of \emph{distinguish}), \verb*|prehensive| (suffix of
    \emph{comprehensive}), \verb*|ifying| (as in \emph{simplifying/identifying/verifying}), \verb*|ifies|
    (as in \emph{simplifies/verifies}), \verb*|Alternate|, \verb*|Other|, \verb*|task|, \verb*|Check|,
    \verb*|confirmation|, \verb*| central|, \verb*| Next|, \verb*| best|, \verb*| final|, \verb*|Final|,
    \verb*|Last|, \verb*| first|, \verb*| ways|, \verb*|Long|, \verb*|big|, and \verb*|Com| (start of
    \emph{Compute/Compare/Combine/Complex}).

    \item \textbf{Function words / grammatical glue (14, $\bar{S}=0.066$).}
    \verb*| We|, \verb*|The|, \verb*| the|, \verb*| The|, \verb*| This|, \verb*| is|, \verb*| have|,
    \verb*| can|, \verb*| and|, \verb*| to|, \verb*|not|, \verb*|one|, \verb*|new|, and \verb*|all|---ordinary
    grammar tokens.

    \item \textbf{Math / \LaTeX{} / symbols / delimiters / formatting (27, $\bar{S}=0.066$).}
    \verb*| $| (inline math after a space), \verb*|$| (math delimiter), \verb*| $\| (space\,$+$\,\verb*|$|\,$+$\,command
    start), \verb*| \| (space\,$+$\,backslash), \verb*|\| (command escape), \verb*|=|, \verb*|cal| (as in
    \verb*|\mathcal| or ``calculate''), \verb*|begin| (as in \verb*|\begin{align}|), \verb*|Math|,
    \verb*|frac| (\verb*|\frac|), \verb*|cdot| (\verb*|\cdot|), \verb*|.sqrt| (square-root fragment),
    \verb*|x|, \verb*|}|, \verb*|)|, \verb*|{|, \verb*|](| (markdown link boundary), \verb*!{|! (table/wiki
    delimiter), \verb*|{$| (brace\,$+$\,math delimiter), \verb*|\n| (newline), \verb*| ##| (markdown
    heading), \verb*|.|, the space token, \verb*|,|, \verb*| **| (markdown bold), \verb*|`\n| (backtick
    then newline), and the middle-dot token \texttt{\textperiodcentered}.

    \item \textbf{Math structure / math context (16, $\bar{S}=0.012$).}
    \verb*|let| (``let $x=\dots$''), \verb*| given| (problem-condition language), \verb*| quadr|
    (\emph{quadratic/quadrilateral/quadrant}), \verb*| hex| (\emph{hexagon/hexadecimal}), \verb*|ponents|
    (suffix of \emph{exponents}), \verb*|ultiply| (suffix of \emph{multiply}), \verb*|aring| (as in
    \emph{squaring}), \verb*|triangle|, \verb*|positive|, \verb*|line|, \verb*| foot| (unit/geometry),
    \verb*|ector| (as in \emph{vector/bisector/sector}), \verb*| Extended| (Extended Euclidean
    Algorithm), \verb*|AC| (geometry notation, e.g.\ side $AC$), \verb*|Im| (imaginary part, or start of
    ``Important''), and \verb*|AY| (geometry label, or URL/name noise).

    \item \textbf{Proper names / acronyms / geography / source artifacts (8, $\bar{S}=0.005$).}
    \verb*|L| (from \L{}), \verb*| America| (``Mathematical Association of America''),
    \verb*|J|, \verb*|Ace|, \verb*|AL|, \verb*|TH|, \verb*|external|, and \verb*| external|.

    \item \textbf{Numbers (13, $\bar{S}=0.004$).}
    \verb*|2|, \verb*|21|, \verb*|1|, \verb*|3|, \verb*|453|, \verb*|228|, \verb*|669|, \verb*|734|, \verb*|813|,
    \verb*|355|, \verb*|415|, \verb*|424|, and \verb*|406|---mostly literal numerals; the three-digit ones
    are likely dataset-specific artifacts, examples, ids, or intermediate values.

    \item \textbf{Subword fragments (15, $\bar{S}<0.001$).}
    \verb*|qu| (\emph{quadratic/quantity/equation/sequence}), \verb*|ens| (\emph{tens/dimensions}),
    \verb*|ast| (\emph{last/least/past}), \verb*|ics| (\emph{mathematics/statistics}), \verb*|ct|
    (\emph{factor/function/product}), \verb*|ze| (\emph{zero/analyze/finalize}), \verb*|ant|
    (\emph{constant/discriminant/quadrant}), \verb*|ub| (\emph{subtract/cube/subset}), \verb*|ja| (as in
    \emph{adjacent}), \verb*|istrib| (\emph{distribute/distribution}), \verb*|ner| (\emph{inner/corner}),
    \verb*|ee| (\emph{three/degree/between}), and \verb*|il| (\emph{similar/probability}).  Two further
    fragments are notable: \verb*|oted| (most math-relevant as \emph{denoted}) and \verb*|ric|
    (\emph{geometric/metric/symmetric/trigonometric}).
\end{itemize}

\subsection{Suppressed token families}
\label{app:accuracy-suppressed-token-families}

The suppressed side ($S(y)<0$) contains $5{,}602$ tokens.  These tokens are not ``wrong'' in isolation.
Under the log-indicator accuracy reward, they receive little positive support from correct rollouts
relative to the promoted answer-surface tokens and the smoothing/centering baseline.  After softmax
normalization this appears as a broad negative side of the realized intervention score.  The lower tail
is therefore best read as an anti-drift pattern learned by the accuracy-trained bias, rather than as
token-level proof that every suppressed token causes incorrectness.  The
strongest negative tail includes \verb*| Question|, \verb*| Examples|, \verb*| Explanation|,
\verb*| Approach|, \verb*| Lesson|, markdown headings, and page/template fragments.  A useful first cut
is by strength (Table~\ref{tab:accuracy-score-bands}); we therefore avoid over-interpreting tiny
negatives near zero, as the meaningful story comes from the lower tail.

Following the grouping procedure in Appendix~\ref{app:embedding-clusters}, we group the suppressed tokens
into twelve interpretive families.  The summaries below describe each family, ordered by mean intervention
score $\bar{S}$ (most strongly suppressed first); the complete per-token
listings with intervention scores are provided in the code repository under
\texttt{results/accuracy\_suppressed\_all\_families.tex}.

\begin{itemize}
    \item \textbf{Verbose reasoning wrappers / discourse markers (44, $\bar{S}=-0.73$).}
    Tokens such as \verb*| Moreover|, \verb*| Clearly|, \verb*| Furthermore|,
    \verb*| However|, and \verb*| Therefore| often introduce extended explanatory
    scaffolding or discourse transitions rather than the final answer itself.
    Their suppression suggests that the terminal reward favors reaching the answer directly
    over producing lecture-style reasoning.

    \item \textbf{Question-restatement / prompt-template tokens (40, $\bar{S}=-0.73$).}
    Strongly suppressed tokens such as \verb*| Question| usually start a new prompt or open a Q/A template
    instead of solving the current problem, so the shift discourages worksheet, course-page, and
    new-problem mode.  (Some metadata tokens---\verb*| Keywords|, \verb*|Basic|, \verb*| Vocabulary|---were
    instead \emph{promoted}: those correlated with scraped correct-solution pages, whereas
    \verb*| Question|/\verb*| Examples|/\verb*| Lesson| correlate with drifting into extra page content.)

    \item \textbf{Names, initials, usernames, source/person artifacts (294, $\bar{S}=-0.30$).}
    First names and initials such as \verb*| Alice|, \verb*| Brian|, \verb*| Julie|, and \verb*| Kevin|
    are problem-specific rather than generally useful; suppressing them reduces hallucinated story details
    and invented diagrams.

    \item \textbf{Educational-site / AoPS / forum / wiki artifacts (44, $\bar{S}=-0.19$).}
    Unlike the literal web markup of the HTML / URL / markdown / code family, these are the
    \emph{natural-language} page furniture
    of the math sites in pretraining---contest wikis, Art-of-Problem-Solving threads, and Q\&A
    forums---the kind of boilerplate and page scaffolding that pervades web-scraped pretraining
    corpora~\citep{penedo2024fineweb}.  Since worked solutions there sit inside such pages, the base model
    tends to continue past the answer into the surrounding scaffolding (``Posted by\dots'', ``Reply'',
    ``edited'', ``References'', user/wiki links).  Suppressing these tokens keeps the completion on the
    current problem instead of rolling over into copied page tails.

    \item \textbf{Math-topic labels suppressed despite being mathematical (73, $\bar{S}=-0.14$).}
    The most subtle group: tokens such as \verb*| Coordinate|, \verb*| Distance|, \verb*| Equation|, and
    \verb*| calculus| are meaningful, but often appear as topic labels or wrong-route triggers
    and, averaged over the dataset, were less predictive of success than direct answer-oriented tokens.
    General math vocabulary can be negatively correlated with correctness when it signals
    overcomplication or wrong method selection.

    \item \textbf{Function words and ordinary prose (2{,}536, $\bar{S}=-0.13$).}
    Tokens such as \verb*| while|, \verb*| there|, \verb*| because|, and \verb*| their| are
    not a blanket anti-English bias (some function words were promoted); these reflect the longer prose
    style of rambling completions.

    \item \textbf{Subword shards / tokenizer fragments (1{,}525, $\bar{S}=-0.12$).}
    Fragments such as \verb*|ime| (as in \emph{time/prime}), \verb*|reet| (\emph{street}), \verb*|ting|,
    and \verb*|ul| inherit the statistics of the words containing them; they are suppressed because those
    words correlated with lower-quality continuations, not because the fragment itself is meaningful.

    \item \textbf{Geometry labels and diagram-letter clutter (214, $\bar{S}=-0.091$).}
    Arbitrary point/segment labels such as \verb*| AE|, \verb*| OB|, \verb*| II|, and \verb*| III|
    are easy to hallucinate, so this reads as a bias against invented
    diagrams.  (Note \verb*|AC| was \emph{promoted}: token scores are empirical correlations, not semantic
    rules.)

    \item \textbf{HTML / URL / markdown / code artifacts (189, $\bar{S}=-0.077$).}
    Tokens such as \verb*|<a|, \verb*| https|, and \verb*|~~| usually mean the model has left the
    solution and is emitting web/markdown/HTML/source residue,
    which is almost always bad in a final-answer setting.

    \item \textbf{Generic assistant / politeness / meta-answer style (28, $\bar{S}=-0.068$).}
    Conversational filler such as \verb*|Thanks|, \verb*|Please|, \verb*| sorry|, and \verb*| appreciate|
    (``I hope this helps'') rather than answer extraction.

    \item \textbf{Arbitrary numbers and numeric constants (516, $\bar{S}=-0.052$).}
    Globally promoting arbitrary numerals such as \verb*|846|, \verb*|523|, and \verb*|123| would invite
    hallucinated answers, so suppression here is
    healthy: ``do not prefer this particular number unless the context demands it.''

    \item \textbf{Word-problem scenery / concrete objects (99, $\bar{S}=-0.049$).}
    Surface nouns from problem statements such as \verb*| plane|, \verb*| coins|, \verb*| pigs|, and
    \verb*| night|, not solution moves; these should be driven by the actual prompt.
\end{itemize}

\subsection{Interpretation}
\label{app:accuracy-interpretation}

Read together, the two sides describe a steering vector that favors a clean, staged path to the answer.
The promoted side raises the surface and discourse regularities of correct, well-formatted solutions:
layout and delimiters; the headers, connectives, and solution verbs that scaffold explicit
step-by-step working; digits and symbolic notation; and decisive termination.  It thus gently favors solutions that lay out their
reasoning rather than emit a bare value, without choosing the problem-specific reasoning path; the
promoted digits only supply the surface for calculations and do not imply the bias knows which number
is correct.

The suppressed side acts as an anti-drift prior against the recurring failure modes in
Appendix~\ref{app:accuracy-suppressed-token-families}: restating or generating a new question and slipping
into lesson/tutorial mode; padding with verbose discourse markers; emitting scraped web/page residue;
lapsing into generic assistant politeness; inventing problem-specific details; and overcomplicating with
topic labels.  The surprising suppression of genuine math vocabulary should not be read as ``avoid math'':
averaged over the dataset, those tokens more often started verbose, topic-heavy, website-like, or
problem-specific continuations than the direct, answer-oriented tokens the shift prefers.
The code repository lists all $5{,}602$ suppressed tokens with their intervention scores under
\texttt{results/accuracy\_suppressed\_all\_families.tex}.

\section{Qualitative analysis of the compression logit-bias vector}
\label{app:compression-delta-analysis}

This appendix examines what the compression intervention score $S$ (Appendix~\ref{app:intervention-score})
encodes under the \emph{compression reward}.  Of the Llama vocabulary ($128{,}256$ tokens), $4{,}164$ are
visited---$141$ \emph{promoted} ($S>0$) and $4{,}023$ \emph{suppressed} ($S<0$); leading spaces are
preserved exactly (\verb*| use| and \verb*|use| are distinct tokens), and each side is grouped into
semantic families as in Appendix~\ref{app:promoted-cluster-calculation}, with all $4{,}023$ suppressed
tokens listed in the code repository under \texttt{results/compression\_suppressed\_all\_families.tex}.

The compression score is a \emph{two-sided length-control prior}: it promotes tokens that help stop, shift
to a boundary, package a short answer, or express mathematics compactly, and it suppresses tokens that
usually start explanations, checks, examples, long formatting, page tails, or multi-step derivations.
The positive side is small and sharp while the negative side is broad
(Table~\ref{tab:compression-score-bands}), so the bias acts mainly by discouraging length-expanding
continuations while also boosting a few stop, answer, and compact-format tokens.

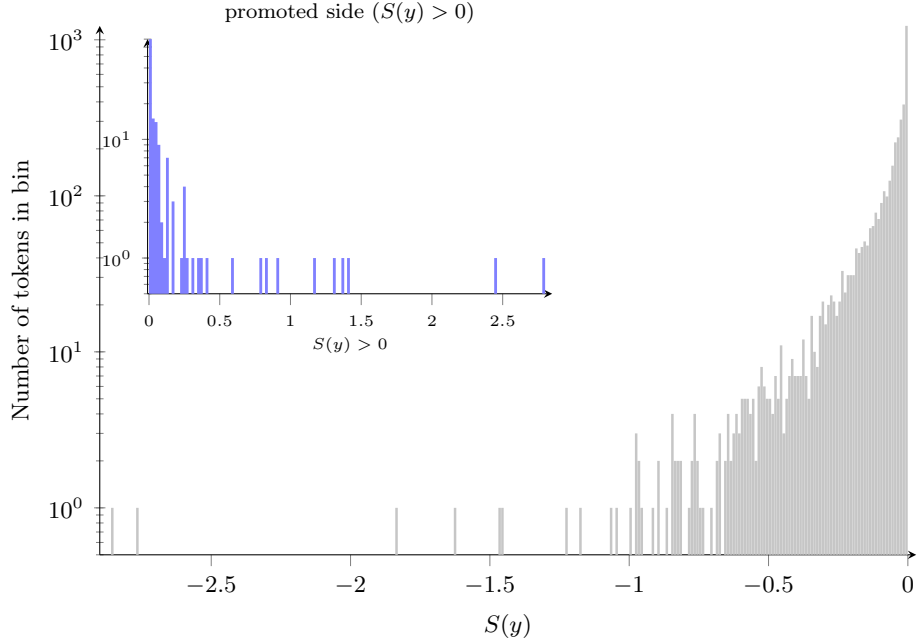
\begin{figure}[t]
  \centering
  \begin{tikzpicture}
    \begin{axis}[
      name=mainax,
      width=0.75\textwidth,
      height=0.52\textwidth,
      ymode=log,
      log origin=infty,
      ymin=0.5,
      xmin=-2.90, xmax=0.03,
      axis lines=left,
      xlabel={$S(y)$},
      ylabel={Number of tokens in bin},
      tick label style={font=\footnotesize},
      label style={font=\footnotesize},
      ybar,
      bar width=0.009,
    ]
      \addplot[fill=gray!45, draw=none, bar shift=0pt] table[
        x=bin_center,
        y expr={\thisrow{count}<1 ? 0.5 : \thisrow{count}},
      ] {data/compression_score_suppressed.tsv};
    \end{axis}
    \begin{axis}[
      at={(mainax.north west)}, anchor=north west,
      xshift=18pt, yshift=-5pt,
      width=0.42\textwidth, height=0.30\textwidth,
      ymode=log, log origin=infty, ymin=0.5,
      axis lines=left,
      xmin=-0.01, xmax=2.85,
      tick label style={font=\tiny},
      title={\scriptsize promoted side ($S(y)>0$)},
      title style={yshift=-4pt},
      xlabel={$S(y)>0$}, xlabel style={font=\tiny, yshift=3pt},
      ybar, bar width=0.02,
    ]
      \addplot[fill=blue!50, draw=none, bar shift=0pt] table[
        x=bin_center,
        y expr={\thisrow{count}<1 ? 0.5 : \thisrow{count}},
      ] {data/compression_score_promoted.tsv};
    \end{axis}
  \end{tikzpicture}
  \caption{Histogram of the intervention score $\widehat{S}(y)$ for the compression reward over the
  Llama vocabulary (visited tokens only; $124{,}092$ unvisited tokens with $\widehat{S}(y)=0$ are
  omitted).  The $y$-axis is logarithmic.  The main panel shows the \emph{suppressed} side
  ($\widehat{S}(y)<0$, $4{,}023$ tokens), which carries the bulk of the mass and a long tail
  reaching $\widehat{S}(y)\approx-2.86$; the inset zooms into the \emph{promoted} side
  ($\widehat{S}(y)>0$, $141$ tokens, up to $\widehat{S}(y)\approx2.80$).}
  \label{fig:compression-delta-histogram}
\end{figure}

\begin{table}[h]
  \centering
  \caption{Score bands of the compression intervention score $S(y)$.}
  \label{tab:compression-score-bands}
  {\small
  \begin{tabular}{lcl}
  \toprule
  Score band & Count & Interpretation \\
  \midrule
  $S\ge 1$              & $6$       & strongest promoted tokens \\
  $0.1\le S<1$          & $25$      & strong positive compression signal \\
  $0.01\le S<0.1$       & $47$      & moderate positive signal \\
  $0<S<0.01$            & $63$      & weak positive / noisy \\
  $-0.001\le S<0$       & $657$     & almost neutral negative \\
  $-0.01\le S<-0.001$   & $571$     & weak suppression \\
  $-0.1\le S<-0.01$     & $1{,}727$ & meaningful suppression \\
  $-1\le S<-0.1$        & $1{,}058$ & strong suppression \\
  $S<-1$                & $10$      & strongest negative shifts \\
  \bottomrule
  \end{tabular}}
\end{table}

\subsection{Promoted token families}
\label{app:compression-promoted-token-families}
The promoted side ($S>0$) contains all $141$ tokens with a positive log-probability shift; we list
every one, grouped into families with a short gloss.  Interpretation is collected in
Appendix~\ref{app:compression-interpretation}.

Families are ordered by their mean intervention score $\bar{S}$ (the average of $S(y)$ over the family's tokens).

\begin{itemize}
    \item \textbf{Stop / boundary tokens (6, $\bar{S}=0.61$).}
    \verb*!<|end_of_text|>!, \verb*| Problem|, \verb*|Problem|, \verb*|\n|, \verb*|:\n|, and
    \verb*| ##|. \verb*!<|end_of_text|>! and \verb*| Problem| are configured stop strings:
    \verb*!<|end_of_text|>! is the direct end-of-text token, and \verb*| Problem| terminates the
    completion at the next problem header (reaching a new problem effectively ends the current answer).
    The remaining \verb*|Problem|, \verb*|\n|, \verb*|:\n|, and \verb*| ##| are not stop strings but
    compact structural boundaries that let the model finish or segment without opening a long paragraph.
    This is the clearest compression signal: end now, or move to a boundary that makes ending natural.

    \item \textbf{Educational / page / UI / source artifacts (19, $\bar{S}=0.46$).}
    \verb*|Personal|, \verb*|Recommended|, \verb*| video|, \verb*|Basic|, \verb*|Win|, \verb*|Ac|,
    \verb*|User|, \verb*|Skill|, \verb*|Special|, \verb*|J|, \verb*| Ch|, \verb*|Credit|,
    \verb*| Display|, \verb*|black|, \verb*| America|, \verb*| engineering|, \verb*|-person|,
    \verb*| Day|, and \verb*| Mathematics|.  These are scraped-page, UI, source, or metadata artifacts
    rather than compression mechanisms---e.g.\ \verb*|Basic| (``Basic Algebra''), \verb*|User|
    (forum/wiki metadata), \verb*| Ch| (``Chapter''), \verb*| America| (``Mathematical Association of
    America'').  They are most likely correlational shortcuts that appeared near short page snippets or
    early-stop regions; the strongest promoted token being \verb*|Personal| is a warning that the
    intervention also picked up dataset artifacts.

    \item \textbf{Ordering / option / concise navigation (12, $\bar{S}=0.22$).}
    \verb*|Alt|, \verb*|Alternate|, \verb*|Alternative|, \verb*| central|, \verb*| Next|, \verb*|Next|,
    \verb*| first|, \verb*| Ways|, \verb*| best|, \verb*|Long|, \verb*| systematic|, and
    \verb*| changing|.  A mixed group: \verb*| Next|/\verb*|Next|/\verb*| first| are compact sequencing
    tokens and \verb*| best| supports answer selection, while \verb*|Alt|/\verb*|Alternate|/\verb*|Alternative|
    may be option labels or page artifacts (\verb*|Long|, \verb*| systematic|, \verb*| changing| are
    weak positives).

    \item \textbf{Minimal action / solution hooks (9, $\bar{S}=0.16$).}
    \verb*|show|, \verb*| Get|, \verb*| find|, \verb*| use|, \verb*| Use|, \verb*| given|,
    \verb*| Since|, \verb*| Define|, and \verb*|align|.  These are short solution hooks (e.g.\
    \verb*| Get| ``Get $x=3$'', \verb*| given| a compact reference to the conditions, \verb*| Since| a
    short justification starter).  The bias prefers skeletal action hooks over derivational prose
    (the imperative-vs-prose minimal pairs are collected in
    Appendix~\ref{app:compression-interpretation}).

    \item \textbf{Minimal grammar / punctuation scaffold (22, $\bar{S}=0.077$).}
    \verb*| We|, \verb*|The|, \verb*| the|, \verb*| The|, \verb*| A|, \verb*| This|, \verb*| |,
    \verb*| can|, \verb*| have|, \verb*|,|, \verb*| is|, \verb*| and|, \verb*| to|, \verb*| in|,
    \verb*| of|, \verb*| a|, \verb*| my|, \verb*| All|, \verb*|.|, \verb*| "|, \verb*|!)|, and
    \verb*| ...|.  Compression does not strip all prose: it keeps just enough grammar and punctuation
    to emit a compact, well-formed response (\verb*| | is a bare space token, \verb*| ...| an
    ellipsis/short-tail marker).

    \item \textbf{Short final-answer markers (11, $\bar{S}=0.055$).}
    \verb*| Final|, \verb*|Final|, \verb*| final|, \verb*| Answer|, \verb*| answer|, \verb*|Choice|,
    \verb*| lowest|, \verb*| number|, \verb*|Number|, \verb*| useful|, and \verb*| important|.  These
    let the model package a response as a final answer (``Final answer\dots'', ``The answer is\dots'',
    compact multiple-choice output) instead of producing a full worked solution.  (Leading-space forms
    occur inside compact sentences; see Appendix~\ref{app:compression-interpretation} for the
    promoted/suppressed minimal pairs.)

    \item \textbf{Compact math / \LaTeX{} / symbolic / geometry (33, $\bar{S}=0.038$).}
    \verb*|{$|, \verb*| $|, \verb*| $\|, \verb*| quadr|, \verb*!{|!, \verb*|$|, \verb*|begin|,
    \verb*|\|, \verb*|}|, \verb*| hex|, \verb*| \|, \verb*|tan|, \verb*|)|, \verb*| (|, \verb*|frac|,
    \verb*|triangle|, \verb*|AC|, \verb*|{|, \verb*| Square|, \verb*|=|, \verb*|x|, \verb*| -|,
    \verb*| =|, \verb*| Extended|, \verb*|EF|, \verb*|cdot|, \verb*|_C|, \verb*|QR|, \verb*|cs|,
    \verb*|cal|, \verb*|("$|, \verb*|("|, and \verb*|·| (token~115, a raw-byte artifact).  These support
    compact mathematical expression: inline math delimiters (\verb*| $|, \verb*|$|), command escapes
    (\verb*|\|), subword notation (\verb*|frac|, \verb*|cdot|, \verb*|tan|), geometry labels
    (\verb*|AC|, \verb*|EF|, \verb*|QR|), and compact algebra (\verb*|x|, \verb*|=|, \verb*| -|): the
    bias likes inline math, not display-math derivations (see the minimal pairs in
    Appendix~\ref{app:compression-interpretation}).

    \item \textbf{Numbers (9, $\bar{S}=0.011$).}
    \verb*|2|, \verb*|21|, \verb*|1|, \verb*|3|, \verb*|996|, \verb*|163|, \verb*|926|, \verb*|222|,
    and \verb*|354|.  Small digits (\verb*|1|, \verb*|2|, \verb*|3|) are useful in concise answers and
    enumeration; the larger numbers are weak positives and probably dataset-specific rather than
    generally meaningful.

    \item \textbf{Subword fragments / morphemes (20, $\bar{S}<0.001$).}
    \verb*|stit|, \verb*|cept|, \verb*|ector|, \verb*|oid|, \verb*|-negative|, \verb*|ilateral|,
    \verb*|istinguish|, \verb*|ate|, \verb*|om|, \verb*|ifying|, \verb*|ifies|, \verb*|ues|,
    \verb*|ithmetic|, \verb*|ee|, \verb*|_r|, \verb*|ultiply|, \verb*|itions|, \verb*|aring|,
    \verb*|oted|, and \verb*|anging|.  Tokenizer fragments that inherit meaning from the words they
    appear in (e.g.\ \verb*|ector| as \emph{vector/bisector}, \verb*|ilateral| as
    \emph{quadrilateral/equilateral}, \verb*|ithmetic| as \emph{arithmetic}); treat them as
    low-confidence weak positives.
\end{itemize}

\subsection{Suppressed token families}
\label{app:compression-suppressed-token-families}
The suppressed side ($S<0$) contains $4{,}023$ tokens and is best read as an anti-continuation prior:
these tokens are not ``wrong'' in isolation, but they tend to predict that the answer is about to get
longer.  The summaries below describe each family, ordered by mean intervention score $\bar{S}$ (most
strongly suppressed first); the complete per-token listings with intervention
scores are provided in the code repository under
\texttt{results/compression\_suppressed\_all\_families.tex}.

\begin{itemize}
    \item \textbf{Multiline formatting / display math / markdown expansion (15, $\bar{S}=-0.55$).}
    A very clean compression signal: multi-line formatting (\verb*| \n\n|), display math (\verb*| $$|),
    deeper headings (\verb*| ###|), and long spacing all predict a longer answer, whereas a single
    boundary or inline formula is compact.

    \item \textbf{Course-page / website / practice-mode tokens (29, $\bar{S}=-0.54$).}
    Among the strongest negatives (\verb*| Latest|, \verb*| Practice|, \verb*| Community|): page
    navigation, educational-site residue, and scraped tails that predict the model has moved away from
    the answer into website content.

    \item \textbf{Explanation / solution-section headings (18, $\bar{S}=-0.47$).}
    Tokens such as \verb*| Explanation|, \verb*| Solution|, \verb*| Approach|, and \verb*|Definition| are
    useful for tutoring but they tend to open a new section or paragraph, so compression penalizes
    entering explanation mode---a clear difference from the accuracy reward, which may keep such
    scaffolding.

    \item \textbf{Step-by-step derivation starters (27, $\bar{S}=-0.36$).}
    These begin visible reasoning (\verb*|Using|, \verb*|Given|, \verb*| Finding|), so the model is
    pushed away from ``show the work'' mode.

    \item \textbf{Checking / verification / self-correction (26, $\bar{S}=-0.32$).}
    Verification aids reliability but costs tokens (it usually follows the answer), so the bias
    suppresses it---at the risk of removing useful sanity checks.  \verb*| Check| is especially
    suppressed.

    \item \textbf{Names, people, and story nouns (170, $\bar{S}=-0.18$).}
    Narrative detail from problem statements such as \verb*| Lisa|, \verb*| Natalie|, and \verb*| Point|;
    a short answer should not invent people, holidays,
    objects, or scenery.

    \item \textbf{Programming / code / system-ish artifacts (23, $\bar{S}=-0.13$).}
    Code-like continuations such as \verb*| Java|, \verb*| Python|, \verb*| program|, and \verb*| Test|
    usually lengthen the output and drift from a compact final answer.

    \item \textbf{Math-topic tokens that invite derivation (85, $\bar{S}=-0.12$).}
    These are real math tokens; suppression does not mean ``avoid math'' but that they tend to open a
    topic-specific derivation or calculation, which is costly under a shorter-answer reward
    (\verb*| integration| is the clearest case).

    \item \textbf{Discourse connectives and paragraph extenders (1{,}669, $\bar{S}=-0.098$).}
    Tokens such as \verb*| But|, \verb*| Also|, \verb*| And|, and \verb*| Specifically| imply
    the answer is continuing---another sentence, case, or justification---an almost
    anti-chain-of-thought surface signature.

    \item \textbf{Subword fragments (1{,}218, $\bar{S}=-0.067$).}
    Fragments whose larger words tend to start longer explanations or theorem references
    (\verb*| Squ| \emph{Square}, \verb*| Con| \emph{Consider}, \verb*| Ferm| \emph{Fermat}).

    \item \textbf{Geometry labels / arbitrary point names / initials (203, $\bar{S}=-0.061$).}
    Some labels are promoted (\verb*|AC|, \verb*|EF|, \verb*|QR|) while others are suppressed
    (\verb*| OA|, \verb*| AB|, \verb*| OB|)---an
    empirical, not semantic, split; suppression helps avoid invented geometry scaffolding.

    \item \textbf{Arbitrary numbers (540, $\bar{S}=-0.023$).}
    A compression prior should not globally promote random constants such as \verb*|388|, \verb*|522|,
    and \verb*|540|; suppressing them reduces
    hallucinated numeric detail (the robust positive numeric signal is only for small digits).
\end{itemize}

\subsection{Interpretation}
\label{app:compression-interpretation}

A handful of contrasts are more informative than individual tokens, because they isolate the
length-control axis from token identity:
\begin{itemize}
    \item \verb*!<|end_of_text|>!, \verb*| Problem|, \verb*|Problem| promoted: stop / next-problem
    boundary behavior.
    \item \verb*|\n| promoted but \verb*| \n\n| suppressed: one boundary is compact, paragraph
    expansion is not.
    \item \verb*| ##| promoted but \verb*| ###| suppressed: a shallow boundary can help stop, deeper
    sectioning extends.
    \item \verb*|$|/\verb*| $| promoted but \verb*| $$| suppressed: inline math is compact, display
    math invites derivation.
    \item \verb*| Answer| promoted but \verb*|Answer| suppressed: a compact answer phrase beats
    heading/template mode.
    \item \verb*| Use| promoted but \verb*|Using| suppressed: an imperative hook beats a prose
    derivation starter.
    \item \verb*|Alternate| promoted but \verb*| Alternate| suppressed: a compact label beats an
    ``alternate solution'' continuation.
    \item \verb*| Check| suppressed: verification is useful but length-expanding.
    \item \verb*| integration| suppressed: math-topic words can be anti-compression when they open
    long derivations.
    \item \verb*| Latest|, \verb*| Practice|, \verb*| Community| strongly suppressed:
    page/navigation/practice tails are anti-answer.
\end{itemize}

Overall, the compression intervention learns a \emph{terminal, compact-answer style}: unlike the
accuracy bias it is not merely promotional but actively pushes down tokens that predict continuation.
Its main risk is therefore \emph{over-compression}, since useful verification, necessary derivation, and
explanatory clarity can be suppressed along with genuine verbosity (Section~\ref{sec:limitations}).
The code repository lists all $4{,}023$ suppressed tokens with their intervention scores under
\texttt{results/compression\_suppressed\_all\_families.tex}.

\section{Additional experimental details}
\label{app:exp-details}

\subsection{Prompt templates and GPQA formatting}
We use the default lm-evaluation-harness prompt templates for all tasks. In particular, the core format follows the standard task prompt pattern (e.g., ``Problem: \ldots Answer: \ldots'').

For GPQA-main, we format prompts so the model first provides an explanation and then outputs a final answer choice. We also shuffle the answer choices per example to avoid systematic bias toward any fixed label. For bias estimation, sampled positions are taken from the explanation-first completion before the final answer token sequence.

\subsection{Rollout-generation settings}
For bias estimation, we use dataset-specific rollout budgets and context lengths: MATH uses 64 rollouts per problem with train/test caps of 1024/1024 tokens and 128 sampled positions per rollout; GSM8K uses 128 rollouts with 512/1024 caps and 256 positions; GPQA uses 128 rollouts with 512/4096 caps and 128 positions.

\subsection{Tilting parameter $\tau$ and additive smoothing $\alpha$}
We tune the additive-smoothing constant $\alpha$ (defined in Section~\ref{sec:param}) on validation.

For accuracy experiments, the log-indicator reward $r_{\text{log-accuracy}}\in\{-\infty,0\}$ makes $\tau$ inert: $e^{r/\tau}\in\{0,1\}$ for any $\tau>0$, so the IPS estimator in Equation~\eqref{eq:Z-hat} reduces to an accuracy-weighted count and the tilt parameter drops out. We therefore only report $\alpha$: MATH uses $\alpha=0.052$, GSM8K uses $\alpha=0.110$, and GPQA uses $\alpha=0.007$.
For the compression experiment, Gemma uses $\alpha=0.738$ with $\tau=1.1$, Llama uses $\alpha=0.052$ with $\tau=0.7$, and Qwen uses $\alpha=0.465$ with $\tau=0.8$.

Sweeping $(\tau,\alpha)$ during model selection is inexpensive because the estimator factorizes as $\widehat{\mathbf Z}=\mathbf v^{1/\tau}\mathbf M$ (Proposition~\ref{prop:ips-matrix-form}): the rollout statistics $\mathbf M$ are computed once, and each $(\tau,\alpha)$ configuration is then obtained by a single reward-weighted matrix--vector product followed by the additive-smoothing step, without rescanning the rollouts.

\subsection{Cluster-tied baseline}
The cluster-tied baseline builds $K{=}32$ clusters from whitespace-split answer tokens and ties one bias value per cluster. Concretely, we run $k$-means on token/word representations from the answer corpus, then assign each vocabulary token to its nearest centroid and share the same learned bias within each cluster. This yields a low-parameter tied-bias baseline used for fairness comparisons against full-vocabulary logit bias and LoRA-GRPO.

\subsection{LoRA-GRPO baseline}
LoRA-GRPO uses rank $r=16$, $\alpha=32$, dropout $0.05$, target modules 
\texttt{q\_proj}, \texttt{k\_proj}, \texttt{v\_proj}, \texttt{o\_proj}, \texttt{gate\_proj},\texttt{up\_proj}, \texttt{down\_proj}, bf16, and gradient checkpointing, trained for one epoch. For accuracy experiments, the best checkpoint is selected among 3 validation evaluations.

Dataset-specific settings are:
\begin{itemize}
  \item GPQA: learning rate $5\times10^{-5}$, gradient accumulation 16, max completion 1024, eval max new tokens 4096.
  \item MATH/GSM8K: learning rate $2\times10^{-4}$, gradient accumulation 8, max completion 1024 for MATH and 512 for GSM8K, eval max new tokens 1024.
\end{itemize}

\subsection{Compression experiment}
The compression experiment evaluates three model identifiers:
\texttt{google/gemma-2-9b-it}, \\
\texttt{meta-llama/Meta-Llama-3.1-8B-Instruct}, and
\texttt{Qwen/Qwen3-4B}.

For length-aware LoRA-GRPO, we train one epoch on MATH with objective mode \texttt{corr*T/length}, $T{=}1024$, learning rate $2\times10^{-4}$, batch size 2, gradient accumulation 8, and LoRA $(r,\alpha,\text{dropout})=(16,32,0.05)$. We keep the best checkpoint among 32 evaluations.

For both logit bias and LoRA-GRPO in the compression experiment, model selection chooses the checkpoint or $\tau$-$\alpha$ pair that minimizes mean completion length in tokens among configurations whose accuracy is the closest to the base model.

The compression reward for the logit-bias method is $r_{\mathrm{length}}$, first defined in Equation~\eqref{eq:r-length} (Section~\ref{sec:experimental-setup}):
\[
r_{\mathrm{length}}(y)=\ln\!\left(\frac{T}{\operatorname{length}(y)}\right).
\]

\subsection{Compute resources and minor evaluation variance from batch chunking}

Experiments were run on multiple GPU setups, including RTX 4090, RTX 5090, A100, H100, and H200. The code supports multi-GPU execution.
During evaluation, prompts are batched with \texttt{padding=True} and left padding, so tensor shapes depend on which examples are grouped together. Changing \texttt{CONFIG\_EVAL\_CF\_WORKER\_CHUNK\_SIZE} changes how validation problems are partitioned across workers, which in turn changes batch composition and padding shapes.

Because evaluation uses fp16 on GPU, different batch shapes can induce small differences in floating-point execution order inside GPU kernels. These tiny numerical changes can slightly perturb logits. In borderline cases where two next-token logits are nearly tied, this can flip the greedy \texttt{argmax}, changing the generated completion and, for a small number of examples, accuracy.
Subject to the caveat above, the reported results are reproducible on a 4$\times$H200 machine with \texttt{CONFIG\_EVAL\_CF\_BATCH\_SIZE=128} and \texttt{CONFIG\_EVAL\_CF\_WORKER\_CHUNK\_SIZE=0}.

\subsection{Bootstrap confidence intervals}
\label{app:bootstrap}
We report 95\% percentile bootstrap confidence intervals for all evaluation metrics.
The procedure is the same in every experiment; only the resampling unit and the per-replicate statistic differ between the accuracy (non-paired bootstrap) and compression (paired bootstrap) settings.
In each case we draw \(B=10{,}000\) bootstrap replicates with fixed random seeds for reproducibility.
Each replicate resamples \(n\) evaluation examples with replacement, where \(n\) is the number of evaluated examples in the corresponding evaluation set (Table~\ref{tab:bootstrap_sample_sizes}), computes a replicate statistic, and the reported interval is the pair of empirical \(2.5\)th and \(97.5\)th percentiles of that statistic across the \(B\) replicates.
Here \(B\) controls only the Monte Carlo precision of the percentile estimates; it is not treated as additional data and does not change the statistical sample size, which is the evaluation-set size~\(n\).
Consequently, smaller evaluation sets, such as GPQA-main with \(n=45\), yield correspondingly wider intervals.

\begin{table}[t]
  \centering
  \caption{Evaluation sample sizes used for bootstrap confidence intervals. Each bootstrap replicate resamples \(n\) examples with replacement from the corresponding evaluation set.}
  \label{tab:bootstrap_sample_sizes}
  \begin{tabular}{llcc}
    \toprule
    Experiment & Benchmark & Evaluation examples \(n\) & Resampling unit \\
    \midrule
    Accuracy & MATH & 5000 & example correctness \\
    Accuracy & GSM8K & 1319 & example correctness \\
    Accuracy & GPQA-main & 45 & example correctness \\
    Compression & MATH & 5000 & problem index / paired example \\
    \bottomrule
  \end{tabular}
\end{table}

\subsubsection{Accuracy}
For accuracy we resample over examples. Let
\[
c_i=\mathbf{1}\{m_{x_i}(y^{(i)}_{1:T})=m_{\star,i}\},\qquad i=1,\ldots,n,
\]
denote the correctness indicator for example \(i\).
For replicate \(b\in\{1,\ldots,B\}\), let \(c_i^{*(b)}\) be the \(i\)-th indicator drawn with replacement from \(\{c_i\}_{i=1}^n\), and form the replicate mean
\[
\bar{c}_n^{*(b)}=\frac{1}{n}\sum_{i=1}^n c_i^{*(b)}.
\]
The reported 95\% confidence interval is
\[
\left[
Q_{0.025}\left(\{\bar{c}_n^{*(b)}\}_{b=1}^B\right),
Q_{0.975}\left(\{\bar{c}_n^{*(b)}\}_{b=1}^B\right)
\right],
\]
where \(Q_p\) denotes the empirical \(p\)-quantile.

\subsubsection{Compression}
For comparisons in the compression experiment, the baseline and candidate methods are evaluated on the same examples. Let
\[
d_i=\ell_i^{\mathrm{candidate}}-\ell_i^{\mathrm{baseline}}
\]
be the per-example paired difference in completion length in tokens, or in the relevant numeric metric.
For replicate \(b\in\{1,\ldots,B\}\), let \(d_i^{*(b)}\) be the \(i\)-th paired difference drawn with replacement from \(\{d_i\}_{i=1}^n\), and form the replicate mean
\[
\bar{d}_n^{*(b)}=\frac{1}{n}\sum_{i=1}^n d_i^{*(b)}.
\]
We report the \(2.5\)th and \(97.5\)th empirical percentiles of
\(\{\bar{d}_n^{*(b)}\}_{b=1}^B\) as the 95\% confidence interval for the mean difference.

\section{Unbiasedness of the IPS estimator}
\label{app:finite-sample-est}
\subsection{Matrix form of the IPS estimator}\label{app:ips-matrix-form}

\begin{proposition}\label{prop:ips-matrix-form}
Define
\[
v_n := \exp\!\left(r^{(n)}\right),
\qquad
M_{n,y} := \frac{1}{NP}\sum_{p=1}^{P}
\frac{\mathbf 1\!\left\{y_{t^{(n,p)}}^{(n,p)}=y\right\}}
{\pi_0\!\left(y \mid x^{(n)}, y_{1:t^{(n,p)}-1}^{(n,p)}\right)}.
\]
Then the IPS estimator from Equation~\eqref{eq:Z-hat} satisfies
\[
\widehat Z(y)=\sum_{n=1}^{N} v_n^{1/\tau} M_{n,y},
\]
equivalently,
\[
\widehat{\mathbf Z}=\mathbf v^{\,1/\tau}\mathbf M.
\]
\end{proposition}

\begin{proof}
Indexing each entry of $\mathcal{D}$ by its origin $(n,p)$ in Algorithm~\ref{alg:data-gen}, Equation~\eqref{eq:Z-hat} can be written as
\[
\widehat Z(y)
:=\frac{1}{NP}\sum_{n=1}^{N}\sum_{p=1}^{P}
\frac{\mathbf{1}\!\left\{y_{t^{(n,p)}}^{(n,p)}=y\right\}\,e^{r^{(n,p)}/\tau}}{\pi_0\!\left(y\,\bigm|\,x^{(n,p)},\,y_{1:t^{(n,p)}-1}^{(n,p)}\right)}.
\]
In Algorithm~\ref{alg:data-gen}, each replicated sample satisfies $x^{(n,p)}=x^{(n)}$ and $r^{(n,p)}=r^{(n)}$, so $e^{r^{(n,p)}/\tau}=v_n^{1/\tau}$ is independent of $p$. Factoring this term outside the inner sum gives
\[
\widehat Z(y)=\sum_{n=1}^{N} v_n^{1/\tau}
\left(\frac{1}{NP}\sum_{p=1}^{P}
\frac{\mathbf 1\!\left\{y_{t^{(n,p)}}^{(n,p)}=y\right\}}
{\pi_0\!\left(y \mid x^{(n)}, y_{1:t^{(n,p)}-1}^{(n,p)}\right)}\right)
:=\sum_{n=1}^{N} v_n^{1/\tau} M_{n,y}.
\]
Stacking over $y\in \mathbb{V}$ yields $\widehat{\mathbf Z}=\mathbf v^{\,1/\tau}\mathbf M$.
\end{proof}

Factoring $\widehat{\mathbf Z}=\mathbf v^{1/\tau}\mathbf M$ separates rollout statistics from reward weights.
The matrix $\mathbf M$ depends only on sampled tokens and base-model probabilities, so it can be
computed once; changing the reward transform or temperature only changes $\mathbf v^{1/\tau}$,
making recomputation a single matrix-vector product.

\subsection{Unbiasedness}

\begin{proposition}\label{prop:ips-unbiased}
Let $\widehat{Z}(y)$ be the IPS estimator in Equation~\eqref{eq:Z-hat}, computed on the dataset $\mathcal{D}$ produced by Algorithm~\ref{alg:data-gen}. Assume $\pi_0(y\mid x, y_{1:t-1})>0$ for every prompt $x$ in the support of $\rho$, every $t\in\{1,\dots,T\}$, and every prefix $y_{1:t-1}$ reachable under $p_0(\cdot\mid x)$. Then
\[
  \mathbb{E}[\widehat{Z}(y)] = Z_{\mathrm{avg}}(y),
\]
where the expectation is over the randomness of the estimator.
\end{proposition}

\begin{remark}[Non-i.i.d.\ samples, still unbiased]
When $P>1$, summands sharing a common $n$ are not independent---they share the prompt, trajectory, and reward---but unbiasedness only requires correct termwise means, which the proof verifies. The dependence affects variance, not mean.
\end{remark}

\begin{proof}
By linearity of expectation, it suffices to show that each summand has mean $Z_{\mathrm{avg}}(y)$. Fix $(n,p)$ and condition on $x^{(n,p)}=x$, $t^{(n,p)}=t$. The autoregressive factorization $p_0(y_{1:T}\mid x)=p_0(y_{1:t-1}\mid x)\,\pi_0(y_t\mid x,y_{1:t-1})\,p_0(y_{t+1:T}\mid x,y_{1:t})$, together with the indicator $\mathbf{1}\{y_t=y\}$, cancels the $\pi_0(y\mid x,y_{1:t-1})$ factor in the denominator, so
\[
\mathbb{E}\!\left[\frac{\mathbf{1}\{y_t=y\}\,e^{r(x,Y)/\tau}}{\pi_0(y\mid x,y_{1:t-1})}\,\bigg|\,x,t\right]
= \mathbb{E}_{\substack{y_{1:t-1}\sim p_0(\cdot\mid x)\\y_{t+1:T}\sim p_0(\cdot\mid x,y_{1:t-1},y)}}\!\left[e^{r(x,\,y_{1:t-1},\,y,\,y_{t+1:T})/\tau}\right].
\]
The right-hand side is exactly the interventional $\mathrm{do}(y_t=y)$ expectation that defines $Z_{\mathrm{avg}}(y)$ in Equation~\eqref{eq:Z-avg}: the importance weighting realizes the distribution $p_0(\cdot\mid x,\mathrm{do}(y_t=y))$, drawing the prefix and suffix from the base model with $y$ fixed at position~$t$.
Taking the outer expectation over $x\sim\rho$ and $t\sim\mathrm{Unif}\{1,\dots,T\}$ via the tower rule matches Equation~\eqref{eq:Z-avg}, so each summand has mean $Z_{\mathrm{avg}}(y)$. \qedhere
\end{proof}

\section{Bias Parameterization Algorithms}
\label{app:method-algorithms}

This appendix provides the full algorithmic details for the two bias parameterizations introduced in Section~\ref{sec:method}.
All take as input the dataset $\mathcal{D}$ produced by Algorithm~\ref{alg:data-gen}.

\subsection{Per-token bias}\label{app:per-token-alg}

Algorithm~\ref{alg:sampling} computes the per-token IPS estimator $\widehat{Z}(y)$ from Equation~\eqref{eq:Z-hat} by accumulating, for each sampled position, 
the weight $e^{r/\tau}/\pi_0(y\mid x, y_{1:t-1})$ into the bin of the observed token identity.
 The per-token bias $\widehat{\delta}(y)=\log\widehat{Z}(y)+c_0$ is then obtained by the step in Section~\ref{sec:param}.

\begin{algorithm}[t]
\caption{\textsc{PerTokenIPS}: per-token IPS estimator $\widehat{Z}(y)$}
\label{alg:sampling}
\begin{algorithmic}[1]
\Require dataset \(\mathcal{D}=\{(x^{(n,p)},\,y_{1:T}^{(n,p)},\,t^{(n,p)},\,r^{(n,p)})\}_{n,p=1}^{N,P}\), base model \(\pi_0\), temperature \(\tau\)
\State \(S(\cdot)\gets 0\)
\For{\((x,\,y_{1:T},\,t,\,r)\in\mathcal{D}\)}
    \State \(\pi \gets \pi_0\bigl(y_t\,\bigm|\,x,\,y_{1:t-1}\bigr)\) \Comment{base-model probability of the on-policy token at position $t$}
    \State \(S(y_t)\gets S(y_t) + e^{r/\tau}\,/\,\pi\)
\EndFor
\ForAll{token \(y\) with \(S(y)>0\)}
    \State \(\widehat{Z}(y)\gets S(y)/(NP)\)
\EndFor
\State \Return \(\{\widehat{Z}(y)\}_{y\in \mathbb{V}_+}\)
\end{algorithmic}
\end{algorithm}

\subsection{Cluster-tied bias}\label{app:cluster-procedure}

The cluster-tied parameterization reduces the number of learned bias values by tying together vocabulary tokens with similar string embeddings.  It first builds a vocabulary partition
\(\{C(1),\dots,C(K)\}\), then replaces the per-token IPS estimate by the average estimate within the token's cluster.

Let
\[
  \operatorname{detok}:\mathbb{V}^{*}\to\Sigma^{*}
\]
denote the tokenizer detokenization map, which takes a token \emph{sequence} as input.  For a single token \(y\in\mathbb{V}\), write \((y)\) for the length-one sequence containing it and define its string form as the detokenization of that singleton sequence,
\[
  \operatorname{str}(y):=\operatorname{detok}\bigl((y)\bigr).
\]
Thus \(\operatorname{str}(y)\) is the tokenizer's surface string for the one-token sequence \((y)\).  We cluster the vocabulary token by token, so each token is embedded through its own surface string rather than through any surrounding sentence.

\paragraph{Cluster construction.}
Clusters are built from an auxiliary prompt set
\[
  \mathcal{X}_{c}=\{x_c^{(m)}\}_{m=1}^{N_c},
\]
chosen independently of the rollout prompts used to form \(\mathcal{D}\).  For each auxiliary prompt, we sample a completion
\[
  y_{1:T}^{c,(m)}\sim p_0(\cdot\mid x_c^{(m)})
\]
and detokenize it to a string
\[
  u^{(m)}:=\operatorname{detok}\!\left(y_{1:T}^{c,(m)}\right).
\]
We then form a string corpus
\[
  \mathcal{S}
  :=
  \left(\bigcup_{m=1}^{N_c}\operatorname{words}(u^{(m)})\right)
  \cup
  \mathcal{S}_{\mathrm{special}},
\]
where \(\operatorname{words}(u)\) denotes the set of whitespace-delimited words in \(u\), and
\(\mathcal{S}_{\mathrm{special}}\) contains tokenizer-specific vocabulary strings that should be clustered explicitly, such as EOS, control, whitespace-only, newline, and byte-fallback tokens.

We embed each \(s\in\mathcal{S}\) with \texttt{sentence-transformers/all-mpnet-base-v2} (Sentence-BERT on MPNet; \citealp{reimers2019sentence,song2020mpnet,sentence_transformers_all_mpnet_base_v2}), i.e., \(\mathrm{Encoder}(s)\in\mathbb{R}^d\), run \(K\)-means, and obtain centroids
\(\{\mu_k\}_{k=1}^{K}\).  Each vocabulary token is assigned to its nearest centroid by
\[
  c(y)
  :=
  \arg\min_{k\in\{1,\dots,K\}}
  \left\|\mathrm{Encoder}(\operatorname{str}(y))-\mu_k\right\|_2^2,
  \qquad y\in\mathbb{V}.
\]
This induces the partition
\[
  C(k):=\{y\in\mathbb{V}:c(y)=k\},
  \qquad k=1,\dots,K.
\]

\paragraph{Cluster-tied estimation.}
Given the partition, we first compute the per-token IPS estimates
\(\{\widehat Z(y)\}_{y\in\mathbb{V}_+}\) using Algorithm~\ref{alg:sampling}.  Then each cluster receives the average of its observed token estimates:
\[
  \bar Z(k)
  :=
  \frac{1}{|C(k)\cap\mathbb{V}_+|}
  \sum_{y\in C(k)\cap\mathbb{V}_+}\widehat Z(y).
\]
The tied estimate for a token is
\[
  Z_{\mathrm{tied}}(y):=\bar Z(c(y)).
\]

\begin{algorithm}[t]
\caption{Cluster-tied \(\bar Z\) estimation}
\label{alg:cluster-tied-z}
\begin{algorithmic}[1]
\Require rollout dataset \(\mathcal{D}\), base model \(\pi_0\), temperature \(\tau\)
\Require auxiliary prompts \(\mathcal{X}_{c}\), embedding map \(\mathrm{Encoder}\), number of clusters \(K\)
\State \(\mathcal{S}\gets\emptyset\)
\ForAll{\(x_c\in\mathcal{X}_{c}\)}
    \State Sample \(y_{1:T}^{c}\sim p_0(\cdot\mid x_c)\)
    \State \(u\gets\operatorname{detok}(y_{1:T}^{c})\) \Comment{decode the rollout to a text string}
    \State \(\mathcal{S}\gets\mathcal{S}\cup\operatorname{words}(u)\) \Comment{\(\operatorname{words}(\cdot)\): split into whitespace-delimited words}
\EndFor
\State \(\mathcal{S}\gets\mathcal{S}\cup\mathcal{S}_{\mathrm{special}}\)
\State Fit \(K\)-means to \(\{\mathrm{Encoder}(s):s\in\mathcal{S}\}\), obtaining centroids \(\{\mu_k\}_{k=1}^{K}\)
\ForAll{\(y\in\mathbb{V}\)}
    \State \(c(y)\gets\arg\min_{k}\left\|\mathrm{Encoder}(\operatorname{str}(y))-\mu_k\right\|_2^2\)
\EndFor
\State \(\{\widehat Z(y)\}_{y\in\mathbb{V}_+}\gets\textsc{PerTokenIPS}(\mathcal{D},\pi_0,\tau)\) \Comment{Algorithm~\ref{alg:sampling}}
\For{\(k=1,\dots,K\)}
    \State \(C(k)\gets\{y\in\mathbb{V}:c(y)=k\}\)
    \State \(\bar Z(k)\gets
    \frac{1}{|C(k)\cap\mathbb{V}_+|}
    \sum_{y\in C(k)\cap\mathbb{V}_+}\widehat Z(y)\)
\EndFor
\ForAll{\(y\in\mathbb{V}\)}
    \State \(Z_{\mathrm{tied}}(y)\gets\bar Z(c(y))\)
\EndFor
\State \Return \(\{Z_{\mathrm{tied}}(y)\}_{y\in\mathbb{V}}\)
\end{algorithmic}
\end{algorithm}

\section{Proofs}
\label{app:proofs}

This appendix collects all deferred proofs.
We first establish two technical lemmas (Appendices~\ref{app:softmax-kl}--\ref{app:kl-lower}) that are used throughout, then prove the results from the main text in the order they appear.

\subsection{Softmax--KL bound}\label{app:softmax-kl}

\begin{lemma}\label{lem:softmax-kl}
For all $u,z,z'\in\mathbb{R}^{|\mathbb{V}|}$,
\[
  \operatorname{KL}\!\bigl(\operatorname{softmax}(u+z)\,\|\,\operatorname{softmax}(u+z')\bigr)
  \le \frac{1}{4}\|z-z'\|_2^2.
\]
\end{lemma}

\begin{proof}
Define the log-partition function $\psi(s):=\log\!\bigl(\sum_{a\in \mathbb{V}}e^{u_a+s_a}\bigr)$.
Its gradient is $\nabla\psi(s)=\operatorname{softmax}(u+s)=:p(s)$, and its Hessian is $\nabla^2\psi(s)=\operatorname{Diag}(p(s))-p(s)\,p(s)^\top$.
By Gershgorin's theorem, every eigenvalue of $\nabla^2\psi(s)$ lies in $[0,\tfrac{1}{2}]$: the $i$-th diagonal entry is $p_i(1-p_i)$ and the $i$-th absolute row sum off the diagonal is $\sum_{j\neq i}p_i p_j=p_i(1-p_i)$, so the $i$-th Gershgorin disc is $[0,2p_i(1-p_i)]\subseteq[0,\tfrac{1}{2}]$.
Hence $\psi$ has $\tfrac{1}{2}$-Lipschitz gradient.

By Theorem~2.1.5 of \citet{nesterov2004introductory}, the Bregman divergence of a convex function with $L$-Lipschitz gradient satisfies $D_\psi(z',z)\le \frac{L}{2}\|z'-z\|_2^2$.
With $L=\tfrac{1}{2}$ this gives $D_\psi(z',z)\le\frac{1}{4}\|z'-z\|_2^2$.
It remains to identify $D_\psi$ with the KL divergence.
Let $p=\operatorname{softmax}(u+z)=\nabla\psi(z)$.
Then
\begin{align*}
  D_\psi(z',z)
  &= \psi(z')-\psi(z)-\langle\nabla\psi(z),z'-z\rangle \\
  &= \psi(z')-\psi(z)-\sum_a p_a(z'_a-z_a) \\
  &= \sum_a p_a\bigl[\log p_a - (u_a+z'_a-\psi(z'))\bigr] \\
  &= \sum_a p_a\log\frac{p_a}{\operatorname{softmax}(u+z')_a}
  = \operatorname{KL}\!\bigl(\operatorname{softmax}(u+z)\,\|\,\operatorname{softmax}(u+z')\bigr). \qedhere
\end{align*}
\end{proof}

\subsection{KL lower bound via exponential tilting}\label{app:kl-lower}

\begin{lemma}\label{lem:kl-lower}
Fix a prompt~$x$.
Let $r:\mathcal{X}\times\mathbb{V}^T\to[0,R]$, $\sigma^2_0:=\operatorname{Var}_{Y\sim p_0(\cdot\mid x)}(r(x,Y))$, and $q_s(y_{1:T}\mid x):=p_0(y_{1:T}\mid x)\,e^{s\,r(x,y_{1:T})}/Z_s$ for $s\ge 0$.
Then
\[
  \operatorname{KL}(q_s\|p_0)
  = \int_0^s t\,\operatorname{Var}_{q_t}(r)\,dt
  \;\ge\; \frac{\sigma^2_0}{R^2}\bigl(1-e^{-Rs}(1+Rs)\bigr).
\]
\end{lemma}

\begin{proof}
Let $\Lambda(s):=\log\mathbb{E}_{p_0}[e^{s\,r(x,Y)}]$.
Then $\Lambda'(s)=\mathbb{E}_{q_s}[r]$ and $\Lambda''(s)=\operatorname{Var}_{q_s}(r)$.
Since $\log(q_s/p_0)=sr-\Lambda(s)$, we have $\operatorname{KL}(q_s\|p_0)=s\Lambda'(s)-\Lambda(s)$.
Writing $F(s):=s\Lambda'(s)-\Lambda(s)$, note $F'(s)=s\Lambda''(s)=s\,\operatorname{Var}_{q_s}(r)$ and $F(0)=0$, so $\operatorname{KL}(q_s\|p_0)=\int_0^s t\,\operatorname{Var}_{q_t}(r)\,dt$.

For the lower bound, let $V(t):=\operatorname{Var}_{q_t}(r)$.
Then $V'(t)=\mathbb{E}_{q_t}[(r-\mathbb{E}_{q_t}r)^3]$.
Since $r\in[0,R]$, we have $|r-\mathbb{E}_{q_t}r|\le R$, so $(r-\mathbb{E}_{q_t}r)^3\ge -R\,(r-\mathbb{E}_{q_t}r)^2$, giving $V'(t)\ge -R\,V(t)$ and hence $V(t)\ge \sigma^2_0\,e^{-Rt}$ by Gr\"onwall's inequality.
Therefore $\operatorname{KL}(q_s\|p_0)\ge \sigma^2_0\int_0^s t\,e^{-Rt}\,dt = \frac{\sigma^2_0}{R^2}(1-e^{-Rs}(1+Rs))$.
\end{proof}

\subsection{KL-regularized optimum (Equation~\eqref{eq:qstar})}
\label{app:proof-qstar}

\begin{proof}
Write $\mu:=p_0(\cdot\mid x)$ and $\Omega:=\mathbb{V}^T$.
Introducing a Lagrange multiplier $\lambda$ for the constraint $\sum_{y_{1:T}\in\Omega} q(y_{1:T})=1$:
\[
  \mathcal{L}_\lambda(q)
  = \sum_{y_{1:T}\in\Omega} q(y_{1:T})\!\left(\frac{r(x,y_{1:T})}{\tau} - \log\frac{q(y_{1:T})}{\mu(y_{1:T})}\right) + \lambda\!\left(\sum_{y_{1:T}\in\Omega}q(y_{1:T})-1\right).
\]
Setting $\partial\mathcal{L}_\lambda/\partial q(y_{1:T})=0$ for each $y_{1:T}\in\Omega$:
\[
  \frac{r(x,y_{1:T})}{\tau} - \log\frac{q(y_{1:T})}{\mu(y_{1:T})} - 1 + \lambda = 0
  \qquad\Longrightarrow\qquad
  q(y_{1:T}) = \mu(y_{1:T})\,e^{r(x,y_{1:T})/\tau + \lambda - 1}.
\]
Normalizing: $\sum_{y_{1:T}\in\Omega} q(y_{1:T})=e^{\lambda-1}\sum_{y_{1:T}\in\Omega} \mu(y_{1:T})\,e^{r(x,y_{1:T})/\tau}=e^{\lambda-1}\,Z(x)=1$, so $e^{\lambda-1}=1/Z(x)$.
Substituting back gives $p^\star_\tau(y_{1:T}\mid x) = p_0(y_{1:T}\mid x)\,e^{r(x,y_{1:T})/\tau}/Z(x)$.
Uniqueness follows from strict concavity: $-\operatorname{KL}(q\|\mu)$ is strictly concave in~$q$ (since $q\mapsto q\log q$ is strictly convex) and $\mathbb{E}_q[r]$ is linear.
\end{proof}

\subsection{Optimal per-step policy (Equation~\eqref{eq:pi-star})}
\label{app:proof-pi-star}

\begin{proof}
Fix $t\in\{1,\dots,T\}$, a prompt~$x$, and a prefix~$y_{1:t-1}$.
From Equation~\eqref{eq:qstar}, the joint probability of prefix~$y_{1:t-1}$ and next token~$y_t$ under~$p^\star_\tau$ is
\[
  p^\star_\tau(y_{1:t-1},y_t\mid x)
  = \sum_{y_{t+1:T}} p^\star_\tau(y_{1:T}\mid x)
  = \frac{1}{Z(x)}\sum_{y_{t+1:T}} p_0(y_{1:T}\mid x)\,e^{r(x,y_{1:T})/\tau}.
\]
Factor $p_0$ autoregressively as $p_0(y_{1:T}\mid x)=p_0(y_{1:t-1}\mid x)\,\pi_0(y_t\mid x,y_{1:t-1})\,p_0(y_{t+1:T}\mid x,y_{1:t})$.
The prefix term and $\pi_0(y_t\mid x,y_{1:t-1})$ do not depend on $y_{t+1:T}$, so
\[
  p^\star_\tau(y_{1:t-1},y_t\mid x)
  = \frac{p_0(y_{1:t-1}\mid x)\,\pi_0(y_t\mid x,y_{1:t-1})}{Z(x)}
  \underbrace{\sum_{y_{t+1:T}} p_0(y_{t+1:T}\mid x,y_{1:t})\,e^{r(x,y_{1:T})/\tau}}_{=\,Z(x,\,y_{1:t-1},\,y_t)},
\]
where $Z(x,y_{1:t-1},y_t)$ is the soft value in Equation~\eqref{eq:soft-value}.
Summing over $y_t$ gives the prefix marginal:
\[
  p^\star_\tau(y_{1:t-1}\mid x)
  = \frac{p_0(y_{1:t-1}\mid x)}{Z(x)}\sum_{y'\in \mathbb{V}}\pi_0(y'\mid x,y_{1:t-1})\,Z(x,y_{1:t-1},y').
\]
Dividing the joint by the prefix marginal:
\[
  \pi^\star_\tau(y_t\mid x,y_{1:t-1})
  = \frac{\pi_0(y_t\mid x,y_{1:t-1})\,Z(x,y_{1:t-1},y_t)}{\sum_{y'\in \mathbb{V}}\pi_0(y'\mid x,y_{1:t-1})\,Z(x,y_{1:t-1},y')}.
\]
Since $\pi_0(y\mid x,y_{1:t-1})=e^{\ell_y(x,y_{1:t-1})}/\sum_{y'}e^{\ell_{y'}(x,y_{1:t-1})}$, the base-model normalizer cancels and
\[
  \pi^\star_\tau(y_t\mid x,y_{1:t-1})
  = \frac{e^{\ell_{y_t}(x,y_{1:t-1})+\log Z(x,\,y_{1:t-1},\,y_t)}}{\sum_{y'\in \mathbb{V}}e^{\ell_{y'}(x,y_{1:t-1})+\log Z(x,\,y_{1:t-1},\,y')}}
  = \operatorname{softmax}\!\bigl(\ell(x,y_{1:t-1})+\log Z(x,y_{1:t-1},\cdot)\bigr)_{y_t}. \qedhere
\]
\end{proof}

\subsection{Exact factorization implies fixed-bias optimality (Corollary~\ref{cor:exact-factorization})}
\label{app:proof-exact-factorization}

\begin{proof}
Fix $x$, $t$, and a reachable prefix $y_{1:t-1}$.
By Equation~\eqref{eq:pi-star},
\[
\pi^\star_\tau(\cdot \mid x, y_{1:t-1})
=\operatorname{softmax}\!\bigl(\ell(x,y_{1:t-1})+\log Z(x,y_{1:t-1},\cdot)\bigr).
\]
Under the exact-factorization assumption,
\[
\log Z(x,y_{1:t-1},\cdot)=\delta+c_t(x,y_{1:t-1})\,\mathbf{1},
\]
so
\[
\pi^\star_\tau(\cdot \mid x, y_{1:t-1})
=\operatorname{softmax}\!\bigl(\ell(x,y_{1:t-1})+\delta+c_t(x,y_{1:t-1})\,\mathbf{1}\bigr).
\]
Since $\operatorname{softmax}(u+c\,\mathbf{1})=\operatorname{softmax}(u)$ for any $u$ and scalar $c$,
\[
\pi^\star_\tau(\cdot \mid x, y_{1:t-1})
=\operatorname{softmax}\!\bigl(\ell(x,y_{1:t-1})+\delta\bigr)
=\pi_\delta(\cdot \mid x, y_{1:t-1}).
\]
Thus the fixed-bias policy matches the optimal per-step policy at every reachable prefix.
By Proposition~\ref{prop:equivalence}, equality of all per-step conditionals implies equality of the induced trajectory laws:
\[
p_\delta(\cdot\mid x)=p^\star_\tau(\cdot\mid x). \qedhere
\]
\end{proof}

\subsection{Reward gap (Theorem~\ref{thm:strict})}\label{app:proof-reward-gap}

\begin{proof}
We first establish the bound for a fixed prompt~$x$, then average over~$\rho$.
Write $J_x(q):=\mathbb{E}_{Y\sim q(\cdot\mid x)}[r(x,Y)]$ for the per-prompt reward and $\sigma^2_0(x):=\operatorname{Var}_{Y\sim p_0(\cdot\mid x)}(r(x,Y))$.
The proof combines three ingredients: a KL bound between $p_{\widehat{\delta}}$ and~$p^\star_\tau$, Pinsker's inequality, and a lower bound on the improvement of~$p^\star_\tau$ over~$p_0$.

\medskip\noindent\emph{Step 1: KL bound (per-prompt).}
By Equation~\eqref{eq:pi-star}, $\pi^\star_\tau(\cdot\mid x,y_{1:t-1})=\operatorname{softmax}(\ell(x,y_{1:t-1})+\log Z(x,y_{1:t-1},\cdot))$ and $\pi_{\widehat{\delta}}(\cdot\mid x,y_{1:t-1})=\operatorname{softmax}(\ell(x,y_{1:t-1})+\widehat{\delta})$.
Definition~\ref{def:approx-fact} gives $\inf_c\|\log Z(x,y_{1:t-1},\cdot)-(\widehat{\delta}+c\,\mathbf{1})\|_2\le\varepsilon$ for every reachable prefix.
Since adding $c\,\mathbf{1}$ to a softmax argument does not change the distribution, Lemma~\ref{lem:softmax-kl} yields
\[
  \operatorname{KL}\!\bigl(\pi_{\widehat{\delta}}(\cdot\mid x,y_{1:t-1})\,\|\,\pi^\star_\tau(\cdot\mid x,y_{1:t-1})\bigr)
  \le \frac{\varepsilon^2}{4}
\]
at every reachable prefix.
The chain rule for autoregressive KL~\citep{cover2006elements} gives
\[
  \operatorname{KL}\!\bigl(p_{\widehat{\delta}}(\cdot\mid x)\,\|\,p^\star_\tau(\cdot\mid x)\bigr)
  = \mathbb{E}_{Y\sim p_{\widehat{\delta}}(\cdot\mid x)}\!\left[\sum_{t=1}^T \operatorname{KL}\!\bigl(\pi_{\widehat{\delta}}(\cdot\mid x,y_{1:t-1})\,\|\,\pi^\star_\tau(\cdot\mid x,y_{1:t-1})\bigr)\right]
  \le \frac{T\varepsilon^2}{4}.
\]

\medskip\noindent\emph{Step 2: Per-prompt reward gap via Pinsker.}
By Pinsker's inequality~\citep{pinsker1964information,tsybakov2009nonparametric}, $\operatorname{TV}\!\bigl(p_{\widehat{\delta}}(\cdot\mid x),p^\star_\tau(\cdot\mid x)\bigr)\le\varepsilon\sqrt{T/8}$.
Since $r\in[0,R]$:
\[
  \bigl|J_x(p^\star_\tau)-J_x(p_{\widehat{\delta}})\bigr|
  \le R\cdot\operatorname{TV}\!\bigl(p_{\widehat{\delta}}(\cdot\mid x),\,p^\star_\tau(\cdot\mid x)\bigr)
  \le R\,\varepsilon\sqrt{\frac{T}{8}}.
\]

\medskip\noindent\emph{Step 3: Per-prompt improvement of $p^\star_\tau$ over $p_0$.}
Define the exponential tilting path $q_s(y_{1:T}\mid x):=p_0(y_{1:T}\mid x)\,e^{s\,r(x,y_{1:T})}/Z_s$ for $s\ge 0$, and the log-partition function $\Lambda_x(s):=\log\mathbb{E}_{Y\sim p_0(\cdot\mid x)}[e^{s\,r(x,Y)}]$.
Then $\Lambda_x'(s)=\mathbb{E}_{q_s(\cdot\mid x)}[r]=J_x(q_s)$ and $\Lambda_x''(s)=\operatorname{Var}_{q_s(\cdot\mid x)}(r)$, so
\[
  J_x(q_s)-J_x(p_0)=\int_0^s\operatorname{Var}_{q_t(\cdot\mid x)}(r)\,dt.
\]
Let $V_x(t):=\operatorname{Var}_{q_t(\cdot\mid x)}(r)$.
Since $r\in[0,R]$, Gr\"onwall's inequality gives $V_x(t)\ge \sigma^2_0(x)\,e^{-Rt}$ (cf.\ Lemma~\ref{lem:kl-lower}).
Therefore, setting $s=1/\tau$:
\[
  J_x(p^\star_\tau)-J_x(p_0)\ge\frac{\sigma^2_0(x)}{R}\bigl(1-e^{-R/\tau}\bigr).
\]

\medskip\noindent\emph{Per-prompt bound.}
Combining Steps~2 and~3:
\[
  J_x(p_{\widehat{\delta}})-J_x(p_0) \ge \frac{\sigma^2_0(x)}{R}\bigl(1-e^{-R/\tau}\bigr) - R\,\varepsilon\sqrt{T/8}.
\]

\medskip\noindent\emph{Averaging over prompts.}
Taking $\mathbb{E}_{x\sim\rho}$ of both sides and using linearity:
\[
  J(p_{\widehat{\delta}})-J(p_0)
  \ge \frac{V_0}{R}\bigl(1-e^{-R/\tau}\bigr) - R\,\varepsilon\sqrt{T/8},
  \qquad V_0 := \mathbb{E}_{x\sim\rho}\!\bigl[\sigma^2_0(x)\bigr]. \qedhere
\]
\end{proof}

\end{document}